\pdfoutput=1
\documentclass[11pt]{article}
\usepackage{enumerate}
\usepackage[OT1]{fontenc}

\usepackage{mystyle}

\usepackage[textsize=tiny]{todonotes}

\newcommand{\RN}[1]{
  \textup{\uppercase\expandafter{\romannumeral#1}}
}

\definecolor{ultramarine}{rgb}{0.,0.1,0.9}
\usepackage[frozencache=true,cachedir=minted-cache]{minted} 

\usepackage{paralist}
\usepackage{fullpage}

\usepackage[utf8]{inputenc}
\usepackage[T1]{fontenc}
\usepackage{url}
\usepackage{booktabs}
\usepackage{amsfonts}
\usepackage{nicefrac}
\usepackage{microtype}
\usepackage{xcolor}
\usepackage{listings}
\usepackage{wrapfig}
\definecolor{darkgreen}{RGB}{1,130,32}
\let\oldthanks\thanks
\renewcommand{\thanks}[1]{\let\footnotemark\relax\oldthanks{#1}}
\usepackage{caption}
\usepackage{minted}
\usepackage{fancyvrb}
\RecustomVerbatimCommand{\VerbatimInput}{VerbatimInput}
{fontsize=\footnotesize,
 breaklines=true,
 breakanywhere=true, 
 breaksymbol=,
 frame=single,  
 framesep=0.7em,
 labelposition=topline,
}
\DeclareMathAlphabet{\mathmybb}{U}{bbold}{m}{n}
\newcommand{\1}{\mathmybb{1}}
\definecolor{darkgreen}{RGB}{1,170,32}
\definecolor{lightblue}{RGB}{1,122,190}

\usepackage{undertilde}
\theoremstyle{plain}

\title{Self-Exploring Language Models:\\ Active Preference Elicitation for Online Alignment}
\author{
Shenao Zhang\textsuperscript{1}~~~Donghan Yu\textsuperscript{2}~~~Hiteshi Sharma\textsuperscript{2}~~~Han Zhong\textsuperscript{3}~~~Zhihan Liu\textsuperscript{1}\\
Ziyi Yang\textsuperscript{2}~~~~Shuohang Wang\textsuperscript{2}~~~~Hany Hassan\textsuperscript{2}~~~~Zhaoran Wang\textsuperscript{1}
\\
\small
\textsuperscript{1}Northwestern University~~~~~~~~
\textsuperscript{2}Microsoft~~~~~~~~
\textsuperscript{3}Peking University
}

\begin{document}
\date{}
\maketitle

\begin{abstract}
Preference optimization, particularly through Reinforcement Learning from Human Feedback (RLHF), has achieved significant success in aligning Large Language Models (LLMs) to adhere to human intentions. Unlike offline alignment with a fixed dataset, online feedback collection from humans or AI on model generations typically leads to more capable reward models and better-aligned LLMs through an iterative process. However, achieving a globally accurate reward model requires systematic exploration to generate diverse responses that span the vast space of natural language. Random sampling from standard reward-maximizing LLMs alone is insufficient to fulfill this requirement. To address this issue, we propose a bilevel objective optimistically biased towards potentially high-reward responses to actively explore out-of-distribution regions. By solving the inner-level problem with the reparameterized reward function, the resulting algorithm, named \textit{Self-Exploring Language Models} (SELM), eliminates the need for a separate RM and iteratively updates the LLM with a straightforward objective. Compared to \textit{Direct Preference Optimization} (DPO), the SELM objective reduces indiscriminate favor of unseen extrapolations and enhances exploration efficiency. Our experimental results demonstrate that when fine-tuned on Zephyr-7B-SFT and Llama-3-8B-Instruct models, SELM significantly boosts the performance on instruction-following benchmarks such as MT-Bench and AlpacaEval 2.0, as well as various standard academic benchmarks in different settings. Our code and models are available at \url{https://github.com/shenao-zhang/SELM}.
\end{abstract}

\section{Introduction}
Large Language Models (LLMs) have recently achieved significant success largely due to their ability to follow instructions with human intent. As the defacto method for aligning LLMs, Reinforcement Learning from Human Feedback (RLHF) works by maximizing the reward function, either a separate model \citep{ouyang2022training,bai2022training,gao2023scaling} or reparameterized by the LLM policy \citep{rafailov2024direct,rafailov2024r,azar2023general,zhao2023slic}, which is learned from the prompt-response preference data labeled by humans. The key to the success of alignment is the response \textit{diversity} within the preference data, which prevents reward models (RMs) from getting stuck in local optima, thereby producing more capable language models.

Offline alignment methods \citep{rafailov2024direct,tang2024understanding} attempt to manually construct diverse responses for fixed prompts \citep{cui2023ultrafeedback,ivison2023camels,starling2023}, which, unfortunately, struggles to span the nearly infinite space of natural language. On the other hand, online alignment follows an \textit{iterative} procedure: sampling responses from the LLM and receiving feedback to form new preference data for RM training \citep{ouyang2022training,guo2024direct}. The former step helps explore out-of-distribution (OOD) regions through randomness in sampling. However, in standard online RLHF frameworks, maximizing the expected reward learned from the collected data is the only objective for the LLM, sampling from which often leads to responses clustered around local optima. This passive exploration mechanism can suffer from overfitting and premature convergence, leaving the potentially high-reward regions unexplored.

To address this issue, we propose an active exploration method for online alignment that elicits novel favorable responses. In its simplest form, an optimism term $\alpha \max_y r(x, y)$ is added to the reward-fitting objective (e.g., the negative log-likelihood loss $\mathcal{L}_{\text{lr}}$ on dataset $\mathcal{D}$), resulting in a bilevel optimization objective for the \textit{reward} model $r$: 
\#\label{eq_intro}
\max_r\max_y \alpha r(x, y)-\mathcal{L}_{\text{lr}}(r; \mathcal{D}),
\#
where $\alpha$ is a hyperparameter controlling the degree of optimism. The intuition is illustrated in Figure \ref{urm_illu}. Specifically, minimizing the vanilla reward-fitting loss $\mathcal{L}_{\text{lr}}$ is likely to give a locally accurate RM that overfits the observed data and gets stuck in local minima. Random sampling from this vanilla RM may take a long time to explore the OOD regions that contain the best response. By incorporating the optimism term, we obtain an RM that \textit{both} fits the data well and has a large $\max_y r(x, y)$. This ensures that the greedy response $y_u$ from it is either globally optimal when uncertainty in high-reward regions is eliminated, or potentially good in unexplored areas where $r(x, y_u)$ can be arbitrarily huge due to the relaxed reward-fitting loss. Feedback from humans on these responses $y_u$ can then reduce uncertainty and train a more accurate RM. 

\begin{figure}[H]
\centering
\includegraphics[width=0.5\textwidth]{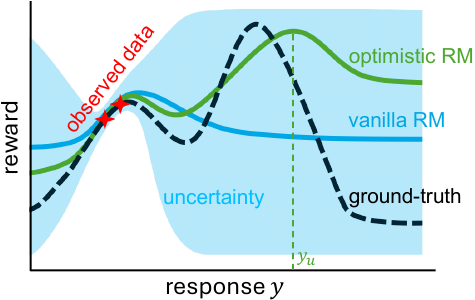}
\caption{Intuition of our method. For a fixed prompt $x$, a reward model $r(x, y)$ tries to fit the ground-truth reward $r^*(x, y)$. The \textcolor{lightblue}{blue} and \textcolor{darkgreen}{green} RMs are equally good when using standard reward-fitting loss $\mathcal{L}_{\text{lr}}$, since the observed preference data (\textcolor{red}{red} stars) are fitted equally well. However, the \textcolor{darkgreen}{green} RM has a larger $\max_y r(x, y)$ and thus a lower optimistically biased loss $\mathcal{L}_{\text{lr}} - \alpha\max_y r(x, y)$. Therefore, the response $y_u$ at which the uncertainty is high can be elicited and then proceeded for human feedback to reduce uncertainty.}
\label{urm_illu}
\end{figure}

In this paper, we formulate this idea within the context of online \textit{direct} alignment, where the LLM is iteratively updated without a separate RM. We first introduce two modifications to the bilevel RM objective in \eqref{eq_intro}, namely adding KL constraints and using relative maximum reward. Then we derive a simple LLM training objective by applying the closed-form solution of the inner-level problem and reparameterizing the reward with the LLM policy. The resulting iterative algorithm is called \textit{Self-Exploring Language Models} (SELM). We show that the policy gradient of SELM is biased towards more rewarding areas. Furthermore, by reducing the chance of generating responses that are assigned low implicit rewards, SELM mitigates the \textit{indiscriminate} favoring of unseen extrapolations in DPO \citep{rafailov2024direct,rafailov2024r} and enhances exploration efficiency. %We also prove that SELM can find an $\varepsilon$-optimal policy within $\tilde{O}(1/\varepsilon^2)$ samples, demonstrating its sample efficiency. 

In experiments, we implement SELM using Zephyr-7B-SFT \citep{tunstall2023zephyr} and Llama-3-8B-Instruct \citep{llama3} as base models. By fine-tuning solely on the UltraFeedback \citep{cui2023ultrafeedback} dataset and using the small-sized PairRM \citep{jiang2023llm} for iterative AI feedback, SELM boosts the performance of Zephyr-7B-SFT and Llama-3-8B-Instruct by a large margin on AlpacaEval 2.0 \citep{dubois2024length} ($+16.24\%$ and $+11.75\%$ LC win rates) and MT-Bench \citep{zheng2024judging} ($+2.31$ and $+0.32$). SELM also demonstrates strong performance on standard academic benchmarks and achieves higher pairwise LC win rates against the very strong iterative DPO baseline, with almost no additional computational overhead under fair comparisons. %Our code and models are available at \url{https://github.com/shenao-zhang/SELM}.
\section{Related Work}
\vspace{-0.15cm}
\paragraph{Data Synthesis for LLMs.} A key challenge for fine-tuning language models to align with users' intentions lies in the collection of demonstrations, including both the SFT instruction-following expert data and the RLHF preference data. Gathering such data from human labelers is expensive, time-consuming, and sometimes suffers from variant quality \citep{ouyang2022training,kopf2024openassistant}. To address this issue, synthetic data \citep{liu2024best} has been used for aligning LLMs. One line of work focuses on generating plausible instruction prompts for unlabeled data by regarding the target output as instruction-following responses \citep{li2023self,wu2023self,josifoski2023exploiting,alpaca,li2024getting}. Besides, high-quality data can also be distilled from strong models for fine-tuning weaker ones \citep{gunasekar2023textbooks,abdin2024phi,li2023textbooks,ding2023enhancing,peng2023instruction}. To construct synthetic datasets for offline RLHF, a popular pipeline \citep{cui2023ultrafeedback,tunstall2023zephyr,wang2024far,ivison2023camels,starling2023} involves selecting responses sampled from \textit{various} LLMs on a set of prompts in the hope to increase the diversity of the data that can span the whole language space. However, data manually collected in such a passive way does not consider what improves the model most through its training, leaving the potentially high-reward regions unexplored. 

\vspace{-0.15cm}
\paragraph{Iterative Online Preference Optimization.} Compared to offline RLHF algorithms \citep{rafailov2024direct,zhao2023slic,azar2023general} that collect preference datasets ahead of training, online RLHF \citep{ouyang2022training,guo2024direct}, especially the iterative/batched online RLHF \citep{bai2022training,xu2023some,chen2022human,gulcehre2023reinforced,snorkelaipair,xiong2023gibbs,calandriello2024human,rosset2024direct} has the potential to gather better and better synthetic data as the model improves. As a special case, self-aligned models match their responses with desired behaviors, such as model-generated feedback \citep{yuan2024self,yuanzhe2024iterative,sun2024principle,wang2024enhancing}. Unfortunately, the above methods still passively explore by relying on the randomness during sampling and easily get stuck at local optima and overfit to the current data due to the vast space of natural language. A notable exception is \cite{dwaracherla2024efficient}, which proposed to use ensembles of RMs to approximately measure the uncertainty for posterior-sampling active exploration. On the contrary, our method explores based on the optimistic bias and does not estimate the uncertainty explicitly, bypassing the need to fit multiple RMs. 

\vspace{-0.15cm}
\paragraph{Active Exploration.} In fact, active exploration has been widely studied beyond LLMs. Similar to \cite{dwaracherla2024efficient}, most existing sample-efficient RL algorithms first estimate the uncertainty of the environment using historical data and then either plan with optimism \citep{auer2002using,russo2013eluder,jin2020provably,mehta2023sample,das2024provably}, or select the optimal action from a statistically plausibly set of values sampled from the posterior distribution \citep{strens2000bayesian,osband2013more,osband2023approximate,zhang2022conservative,li2024hyperagent}. The proposed self-exploration objective can be categorized as an optimism-based exploration method. However, most previous works require the estimation of the upper confidence bound, which is often intractable. Ensemble methods \citep{osband2024epistemic,chua2018deep,lu2017ensemble} can serve as approximations to estimate the uncertainty but are still computationally inefficient. 

\vspace{-0.15cm}
\paragraph{Concurrent Work.} We highlight the concurrent work (to the first version of the current paper) of \cite{xie2024exploratory, cen2024value, liu2024provably}, among which \cite{xie2024exploratory} establishes the first analysis of the sample complexity of a DPO algorithm in the online setting of RLHF (formulated as MDPs). All of them focus on incorporating an SFT loss or a similar term (as bonus or penalty) alongside the DPO loss as an optimistic or pessimistic adjustment in the online or offline setting, respectively. \cite{xie2024exploratory, cen2024value} and the current paper focus on the former, while \cite{liu2024provably} focuses on the latter. In the second version of the current paper, we provide the sample complexity of SELM following the proof technique of \cite{xie2024exploratory}. Through a reduction technique from \cite{xie2024exploratory}, we show how to connect the sample complexity of SELM to that of existing RL algorithms \citep{zhong2022gec, liu2024maximize}, which are not tailored to RLHF but enjoy strong theoretical guarantees. 

%Around the time we posted the first version of our work, several independent studies were also posted. All of these works emphasize the importance of incorporating an SFT loss alongside the DPO loss to act as an optimistic/pessimistic adjustment in the context of online/offline setting. Specifically, \citet{liu2024provably,cen2024value} focuses on the offline setting, while \citet{cen2024value,xie2024exploratory} and our work studies the power of online exploration for RLHF. In our revised version, we add a rigorous sample complexity result, which is inspired by \citet{xie2024exploratory,cen2024value} and classical RL theory literature on GEC \citep{zhong2022gec} and MEX \citep{liu2024maximize}.

%MEX \citep{liu2024maximize} proposed to combine estimation and planning in a single objective similar to ours and established theoretical guarantees under traditional RL setups. RPO \citep{liu2024provably} proposed to use an adversarially chosen reward model for policy optimization, but focuses on mitigating overoptimization in offline settings. 
\section{Background}
%\subsection{Large Language Models}
\paragraph{Large Language Models.} A language model $\pi\in\Delta_\mathcal{Y}^\mathcal{X}$ typically takes the prompt $x\in\mathcal{X}$ as input and outputs the response $y\in\mathcal{Y}$. Here, $\mathcal{X}$ and $\mathcal{Y}$ are finite spaces of prompts and responses, respectively. Given the prompt $x\in\mathcal{X}$, a discrete probability distribution $\pi(\cdot\mid x)\in\Delta_\mathcal{Y}$ is generated, where $\Delta_\mathcal{Y}$ is the set of discrete distributions over $\mathcal{Y}$. After pretraining and Supervised Fine-Tuning (SFT), preference alignment is employed to enhance the ability of the language model to follow instructions with human intentions.
%Modern recipes for training LLMs consist of pre-training and post-training procedures, where during pre-training, LLMs learn to predict the next word on a huge and diverse dataset of text sequences in order to understand the underlying patterns and structures of natural language in a self-supervised manner. The post-training procedure aims to align better to end tasks and human preferences with two phases happening in order: Supervised Fine-Tuning (SFT) and human preference alignment. Here, SFT fine-tunes the pre-trained LLM with supervised learning on high-quality data to follow instructions on downstream tasks and obtain a model $\pi^{\text{SFT}}$. In the following of this paper, we focus mainly on preference alignment.

%\subsection{Reward Modeling and Preference Optimization}
\paragraph{Reinforcement Learning from Human Feedback (RLHF).}
Standard RLHF frameworks consist of learning a reward model and then optimizing the LLM policy using the learned reward. 

Specifically, a point-wise reward $r(x, y): \mathcal{X}\times\mathcal{Y}\rightarrow \mathcal{R}$ represents the Elo score \citep{elo1978rating} of the response $y$ given the prompt $x$. Then the preference distribution can be expressed by the Bradley-Terry model that distinguishes between the preferred response $y_w$ and the dispreferred response $y_l$ given prompt $x$, denoted as $y_w\succ y_l \mid x$, using the logistic function $\sigma$:
\#\label{bt_model}
p(y_w\succ y_l \mid x) &:= \EE_{h}\bigl[\1(h \text{ prefers } y_w \text{ over } y_l \text{ given } x)\bigr] \notag\\
&\,= \sigma\bigl(r(x, y_w) - r(x, y_l)\bigr) = \frac{\exp\bigl(r(x, y_w)\bigr)}{\exp\bigl(r(x, y_w)\bigr) + \exp\bigl(r(x, y_l)\bigr)},
\#
where $h$ denotes the human rater and the expectation is over $h$ to account for the randomness of the choices of human raters we ask for their preference. When provided a static dataset of $N$ comparisons $\mathcal{D}=\{x_i, y_{w,i}, y_{l,i}\}_{i=1}^N$, the parameterized reward model can be learned by minimizing the following negative log-likelihood loss:
\#\label{lr_loss}
\mathcal{L}_{\text{lr}}(r; \mathcal{D}) = -\EE_{(x, y_w, y_l)\sim\mathcal{D}}\bigl[\log\sigma\bigl(r(x, y_w) - r(x, y_l)\bigr)\bigr].
\#

Using the learned reward, the LLM policy $\pi\in\Delta_\mathcal{Y}^\mathcal{X}$ is optimized with reinforcement learning (RL) to maximize the expected reward while maintaining a small deviation from some base reference policy $\pi_{\text{ref}}$, i.e., maximizing the following objective
\#\label{rlhf_kl}
\mathcal{J}(\pi) = \EE_{x\sim\mathcal{D}, y\sim\pi(\cdot\mid x)}\bigl[r(x, y)\bigr] - \beta\mathbb{D}_{\text{KL}}(\pi \,|| \,\pi_{\text{ref}}),
\#
where $\beta$ is a hyperparameter and $\mathbb{D}_{\text{KL}}(\pi \,|| \,\pi_{\text{ref}}):=\EE_{x\sim\mathcal{D}}[\text{KL}(\pi(\cdot \mid x) \,||\, \pi_{\text{ref}}(\cdot \mid x))]$ is the expected Kullback-Leibler (KL) divergence. An ideal $\pi_{\text{ref}}$ is the policy that helps mitigate the distribution shift issue \citep{rafailov2024direct,guo2024direct} between the true preference distribution and the policy $\pi$ during the off-policy RL training. Since we only have access to the dataset $\mathcal{D}$ sampled from the unavailable true preference distribution, $\pi_{\text{ref}}$ can be obtained by fine-tuning on the preferred responses in $\mathcal{D}$ or simply setting $\pi_{\text{ref}}=\pi^{\text{SFT}}$ and performing RLHF based on the SFT model.

\paragraph{Direct Alignment from Preference.} With the motivation to get rid of a separate reward model, which is computationally costly to train, recent works \citep{rafailov2024direct,azar2023general,zhao2023slic,tunstall2023zephyr,ethayarajh2024kto} derived the preference loss as a function of the policy by changing of variables. Among them, DPO \citep{rafailov2024direct} shows that when the BT model in \eqref{bt_model} can perfectly fit the preference, the global optimizers of the RLHF objective in \eqref{rlhf_kl} and the following loss are equivalent:
\$
\mathcal{L}_{\text{DPO}}(\pi;\mathcal{D}) = -\EE_{(x, y_w, y_l)\sim\mathcal{D}}\biggl[\log\sigma\biggl(\beta\log\frac{\pi(y_w\mid x)}{\pi_{\text{ref}}(y_w\mid x)} - \beta\log\frac{\pi(y_l\mid x)}{\pi_{\text{ref}}(y_l\mid x)}\biggr)\biggr].
\$

\section{Self-Exploring Language Models}
\subsection{RM-Free Objective for Active Exploration}
\label{sec_der}
In this section, we present several modifications to the optimistically biased objective \eqref{eq_intro} motivated in the introduction. Then we derive an RM-free objective for the LLM policy and analyze how active exploration works by examining its gradient.

First, we consider the equivalence of \eqref{eq_intro}: $\max_r-\mathcal{L}_{\text{lr}}(r; \mathcal{D}) + \alpha\max_\pi\EE_{y\sim\pi}[r(x, y)]$, where the inner $\pi$ is deterministic when optimal. To account for the change of $\pi$ relative to the reference policy $\pi_{\text{ref}}$, we introduce two modifications: (1) replacing the optimistic bias term $\max_\pi\EE_{y\sim\pi}[r(x, y)]$ with $\max_\pi\EE_{y\sim\pi, y'\sim\pi_{\text{ref}}}[r(x, y)-r(x, y')]$, and (2) incorporating a KL-divergence loss term between $\pi$ and $\pi_{\text{ref}}$. These changes ensure that the resulting optimistic RM elicits responses with high potential unknown to the reference policy $\pi_{\text{ref}}$ while minimizing the deviation between $\pi$ and $\pi_{\text{ref}}$.

Formally, for the reward $r$, the bilevel optimization problem with optimism is formulated as:
\#\label{our_obj}
\max_r  -\mathcal{L}_{\text{lr}}(r; \mathcal{D}_t) +\alpha\max_\pi\Biggl(\underbrace{\EE_{\substack{x\sim\mathcal{D}_t, y\sim\pi(\cdot\mid x)\\ y'\sim\pi_{\text{ref}}(\cdot\mid x)}}\Bigl[r(x, y) - r(x, y')\Bigr] - \beta\mathbb{D}_{\text{KL}}(\pi \,|| \,\pi_{\text{ref}})}_{\mathcal{F}(\pi; r)}\Biggr),
\#
where $\mathcal{D}_t=\{x_i, y_{w,i}^t, y_{l,i}^t\}_{i=1}^N$ is the associated dataset at iteration $t$ and $\mathcal{L}_{\text{lr}}$ is the logistic regression loss defined in \eqref{lr_loss}. The nested optimization in \eqref{our_obj} can be handled by first solving the inner optimization $\mathcal{F}(\pi; r)$ to obtain $\pi_r$ that is optimal under $r$. The solution is as follows and we defer all the derivations in this section to Appendix \ref{derivation_1}.
\$
\pi_r(y\mid x) := \argmax_\pi\mathcal{F}(\pi; r) = \frac{1}{Z(x)}\pi_{\text{ref}}(y\mid x)\exp\bigl(r(x, y) / \beta\bigr), 
\$
where the partition function $Z(x) = \sum_y\pi_{\text{ref}}(y| x)\exp(r(x, y)/\beta)$. By substituting $\pi=\pi_r$ into $\mathcal{F}(\pi; r)$, we can rewrite the bilevel objective in \eqref{our_obj} as a single-level one: \$\max_r -\mathcal{L}_{\text{lr}}(r; \mathcal{D}_t) + \alpha\mathcal{F}(\pi_r; r).\$
Following the implicit reward formulation in DPO, we reparameterize the reward function with $\theta\in\Theta$ as $\hat{r}_\theta(x, y)=\beta(\log\pi_\theta(y\mid x) - \log\pi_{\text{ref}}(y\mid x))$, which is the optimal solution of \eqref{rlhf_kl} and can express \textit{all} reward classes consistent with the BT model as proved in \citep{rafailov2024direct}. With the above change of variable, we obtain the RM-free objective for direct preference alignment with optimism:
\#\label{final_obj}
\max_{\pi_\theta} -\mathcal{L}_{\text{DPO}}(\pi_\theta; \mathcal{D}_t) - \alpha\beta\EE_{x\sim\mathcal{D}, y\sim\pi_{\text{ref}}(\cdot\mid x)}\bigl[\log\pi_\theta(y\mid x)\bigr].
\#
We now analyze how this new objective encourages active exploration. Specifically, we derive the gradient of \eqref{final_obj} with respect to $\theta$ as
\#\label{eq_grad}
&\underbrace{\beta\EE_{(x, y_w, y_l)\sim\mathcal{D}_t}\Bigl[\sigma\bigl(\hat{r}_\theta(x, y_l) - \hat{r}_\theta(x, y_w)\bigr)\bigl(\nabla_\theta\log\pi_\theta(y_w\mid x) - \nabla_\theta\log\pi_\theta(y_l\mid x)\bigr)\Bigr]}_{-\nabla_\theta\mathcal{L}_{\text{DPO}}(\pi_\theta; \mathcal{D}_t)} \notag\\
&\qquad\qquad\qquad\qquad\qquad\qquad - \alpha\beta\EE_{x\sim\mathcal{D}, y\sim\pi_\theta(\cdot\mid x)}\bigl[\exp\bigl(-\hat{r}_\theta(x, y)/\beta\bigr)\nabla_\theta\log\pi_\theta(y\mid x)\bigr].
\#
We note that the second line, corresponding to the gradient of the optimism term, decreases the log-likelihood of response $y$ generated by $\pi_\theta$ that has a high value of $\exp(-\hat{r}_\theta(x, y)/\beta)$. Therefore, the added optimism term biases the gradient toward parameter regions that can elicit responses $y$ with high implicit reward $\hat{r}_\theta$, consistent with our intuition outlined in Figure \ref{urm_illu}.

This also explains why $\EE_{\pi_{\text{ref}}}[\log\pi_\theta]$ is minimized in our objective \eqref{final_obj}, which is equivalent to maximizing the KL divergence between $\pi_{\text{ref}}$ and $\pi_\theta$, while the reverse KL in the policy optimization objective \eqref{rlhf_kl} is minimized. For the DPO gradient $\nabla_\theta\mathcal{L}_{\text{DPO}}(\pi_\theta; \mathcal{D}_t)$, the degree of deviation of policy $\pi_\theta$ from $\pi_{\text{ref}}$ only affects the preference estimated with $\hat{r}_\theta$. In other words, $\sigma(\hat{r}_\theta(x, y_l) - \hat{r}_\theta(x, y_w))$ is a scalar value and the policy deviation only determines the \textit{step size} of the policy gradient, instead of its \textit{direction}. On the other hand, our added exploration term directly controls the direction of the gradient toward potentially more rewarding areas while still fitting the preference data in $\mathcal{D}_t$. As more feedback data is collected iteratively, deviating from the unbiasedly fitted model incurs a higher DPO loss, which ultimately dominates our objective at convergence. This mechanism ensures that the resulting LLM effectively balances between exploring novel responses and exploiting previously observed ones, leading to a more accurate and aligned model.

\subsection{Algorithm}
\label{sec_algo}
With the optimistically biased objective derived above, the language model can actively generate OOD responses worth exploring. Human or AI feedback follows to reduce the uncertainty in these regions. These two steps are executed iteratively to get a more and more aligned model. 

In practice, we split the offline preference dataset into three portions with equal sizes, one for each iteration. Besides, we use AI rankers, such as external RMs, to provide feedback on the model-generated response and the original chosen, rejected responses. The complete pseudocode of our algorithm, named \textit{Self-Exploring Language Models} (SELM), is outlined in Algorithm \ref{alg_se}.
\begin{algorithm}[H]
\caption{Self-Exploring Language Models (SELM)}
\begin{algorithmic}[1]\label{alg_se}
\REQUIRE Reference model $\pi_{\text{ref}}$, preference dataset $\mathcal{D}$, online iterations $T$, optimism coefficient $\alpha$.
\FOR{iteration $t = 1, 2, \ldots, T$}
\STATE Set $\mathcal{D}_{t}$ as the $t$-th portion of $\mathcal{D}$ and generate $y\sim\pi_{\text{ref}}(\cdot\mid x)$ for each prompt $x$ in $\mathcal{D}_t$.
\STATE Rank $\{y, y_w, y_l\}$ and update $\mathcal{D}_t$ to contain the best (chosen) and worst (rejected) responses.
\STATE Train the LLM $\pi_{\theta_t} = \argmax_{\pi_\theta} \{-\mathcal{L}_{\text{DPO}}(\pi_\theta; \mathcal{D}_t) - \alpha \EE_{x\sim\mathcal{D}_t}[\log\pi_{\theta}(y \mid x)]\}$, let $\pi_{\text{ref}}=\pi_{\theta_t}$.
\ENDFOR 
\end{algorithmic}
\end{algorithm}
  
\section{Analysis}
\label{analysis}
\subsection{Self-Exploration Reduces Indiscriminate Favor of Unseen Extrapolations}
It has been observed recently \citep{rafailov2024r,pal2024smaug,xu2024dpo} that DPO decreases the likelihood of responses generated by the reference policy. It is because for any prompt $x$, at convergence when $\pi_\theta \neq \pi_{\text{ref}}$, it holds that
\$\EE_{y\sim\pi_{\text{ref}}}\bigl[\hat{r}_\theta(x, y)/\beta\bigr] = \EE_{y\sim\pi_{\text{ref}}}\bigl[\log\pi_\theta(y\mid x) - \log\pi_{\text{ref}}(y\mid x)\bigr] = -\text{KL}\bigl(\pi_{\text{ref}}(\cdot\mid x) \,||\, \pi_\theta(\cdot\mid x)\bigr) < 0,\$
while at the beginning of training when $\pi_\theta = \pi_{\text{ref}}$, the above terms are zero. Thus, the expected implicit reward $\hat{r}_\theta$ as well as the likelihood of $\pi_\theta$ will decrease on the reference model's responses. 
This indicates that DPO stimulates a biased distribution favoring unseen extrapolated responses. In the online iterative setting that we consider, the LLM policy generates responses and receives preference feedback alternately, where biasing towards OOD regions may sometimes help discover outstanding novel responses. However, DPO \textit{indiscriminately} favors unseen extrapolations and \textit{passively} explores based purely on the randomness inherent in sampling from the LLM. As a consequence, the vast space of natural language makes it almost impossible to exhaustively explore all the possible responses and identify those that most effectively benefit alignment.

Next, we demonstrate that SELM mitigates this issue by performing guided exploration. Specifically, consider the proposed self-exploration objective in \eqref{final_obj}, which, in addition to the standard DPO loss, also minimizes $\EE_{x, y\sim\pi_{\text{ref}}}[\log\pi_\theta(y\mid x)]$. We now investigate how the probability distribution changes with this term incorporated.
\begin{theorem}
\label{thm}
For any $\rho\in\Theta$ in the policy parameter space, let $\hat{r}_\rho(x, y) = \beta(\log\pi_\rho(y\mid x) - \log\pi_{\text{ref}}(y\mid x))$ be the reparameterized implicit reward. Denote $\pi^{\min}_\rho$ as the policy that minimizes the expected implicit reward under the KL constraint, i.e.,
\#\label{eq_pi_rho}
\pi^{\min}_\rho(\cdot\mid x) := \argmin_\pi\EE_{x, y\sim\pi(\cdot\mid x)}\bigl[\hat{r}_\rho(x, y)\bigr] + \beta\mathbb{D}_{\text{KL}}(\pi \,|| \,\pi_\rho).
\#
Then minimizing $\EE_{x, y\sim\pi_{\text{ref}}}[\log\pi_\theta(y| x)]$ decreases the likelihood of responses sampled from $\pi^{\min}_\rho$:
\$
\min_{\pi_\theta}\EE_{x, y\sim\pi_{\text{ref}}(\cdot\mid x)}\bigl[\log\pi_\theta(y\mid x)\bigr] = \min_{\pi_\theta}\EE_{x,y\sim\pi^{\min}_\rho(\cdot\mid x)}\bigl[\log\pi_\theta(y\mid x)\bigr].
\$
\end{theorem}

%\begin{proof}
%See Appendix \ref{proof_indis} for a detailed proof.
%\end{proof}

The proofs for theorems in this section can be found in Appendix \ref{proof_indis} and \ref{proof_thm}. The above theorem states that maximizing the divergence between $\pi_\theta$ and $\pi_{\text{ref}}$ is essentially reducing the probability of generating responses with low implicit rewards reparameterized by any policy parameter $\rho$ during training. In other words, the LLM policy not only exploits the existing preference data but also learns to avoid generating the text $y$ that is assigned a low reward value. This process occurs in every iteration with updated reference models. Consequently, responses with high potential rewards are selectively preferred and many commonplace responses receive a small probability mass, thus mitigating the indiscriminate favoring of unseen responses and improving the exploration efficiency. In the next section, we will formally prove that the self-exploration mechanism is sample-efficient.

\subsection{Self-Exploration is Provably Sample-Efficient}
Following the proof technique of \cite{xie2024exploratory}, we provide the sample efficiency of the self-exploration mechanism by establishing a sublinear cumulative regret. Specifically, the cumulative regret $\mathcal{R}(T)$ up to $T$ iterations is defined as the cumulative performance discrepancy between the learned policy $\pi_t$ at iteration $t$ and the optimal policy $\pi^*$ over the run of the algorithm:
\$
\mathcal{R}(T) & = \sum_{t = 1}^T [\mathcal{J}(\pi^*) - \mathcal{J}(\pi_t)].
\$
The key idea is a reduction technique from \cite{xie2024exploratory}, which connects the sample complexity of SELM to that of existing RL algorithms \citep{zhong2022gec, liu2024maximize}. It is worth noting that the theoretical version of the self-exploration mechanism (Algorithm~\ref{alg_se_theory}) is a bit different from the practical one used in the numerical experiments and is closer to the proposed algorithm in \cite{xie2024exploratory}.
%, which are not tailored to RLHF but enjoy strong theoretical guarantees. 
\begin{assumption}[Realizable Policy Class with Regularity Condition] \label{assumption:policy:class}
    We assume access to a policy class $\Pi$ containing the optimal policy $\pi^*$. Moreover, we assume that
    \$
    \left| \log \frac{\pi(y \mid x)}{\pi_{\mathrm{ref}}(y \mid x)} \right| \le R_{\max}.
    \$
    for any $\pi \in \Pi$ and prompt-response pair $(x, y)$.
\end{assumption}

Assumption \ref{assumption:policy:class} stipulates that the policy class $\Pi$ is sufficiently comprehensive to include the optimal policy. Additionally, it imposes a bounded condition on $\log (\pi/\pi_{\mathrm{ref}})$, which has been identified as the implicit reward function for DPO \citep{rafailov2024direct}.

\begin{theorem}\label{thm:regret}
Under Assumption~\ref{assumption:policy:class}, let $\eta = \sqrt{T d_{\mathrm{PGEC}}/(\exp(4R_{\max}) \log(|\Pi|/\delta)) }$, $\alpha = 2/(\eta \exp(4R_{\max}))$, and $\delta \in (0, 1)$. Then with probability at least $1 - \delta$, we have
    \$
    \mathcal{R}(T) &\lesssim \sqrt{d_{\mathrm{PGEC}} \cdot \exp(2R_{\max}) \cdot T \cdot \log (|\Pi|/\delta)},
    \$
where $\lesssim$ omits absolute constants, and $d_{\mathrm{PGEC}}$ is a preference-based version of Generalized Eluder Coefficient \citep[GEC;][]{zhong2022gec} defined in Appendix~\ref{def:pgec} capturing the complexity of learning problem. For log-linear policy class $\Pi = \{\pi_\theta: \pi_\theta(y \given x) \propto \exp(\la \phi(x, y), \theta\ra/\beta)\}$ with $d$-dimensional feature $\phi$, it holds that $d_{\mathrm{PGEC}} \le \tilde{O}(d)$. 
\end{theorem}
%\begin{proof}
%See Appendix \ref{proof_thm} for a detailed proof.
%\end{proof}

The proof technique is from \cite{xie2024exploratory}, which connects RLHF with RL and allows us to use the preference-based version of GEC \citep{zhong2022gec,liu2024maximize} as the complexity measure to characterize the cumulative regret $\cR(T)$. We restate the proof technique from \cite{xie2024exploratory} for completeness. We emphasize that it is not a novel contribution of the present work. Since the cumulative regret is sublinear in the number of iterations $T$, the above theorem indicates that the policy $\pi_t$ converges to the optimal $\pi^*$ within sufficient iterations. Moreover, by the standard online-to-batch argument, Theorem~\ref{thm:regret} shows that SELM is capable of finding an $\varepsilon$-optimal policy with a sample complexity of $\tilde{O}(1/\varepsilon^2)$. This highlights the sample efficiency of SELM from the theoretical perspective. 
\section{Experiments}
\subsection{Experiment Setup}
We adopt UltraFeedback \citep{cui2023ultrafeedback} as our training dataset, which contains 61k preference pairs of single-turn conversations. For the external ranker during online alignment, we choose the small-sized PairRM (0.4B) \citep{jiang2023llm}. All experiments are conducted on 8xA100 GPUs.

Due to the absence of performant open-source online direct alignment codebases at the time of this study, we first implement an iterative version of DPO as the baseline, adhering to the same steps as Algorithm \ref{alg_se} but training the LLM with the standard DPO objective. Then we conduct a grid search over hyperparameters, such as the batch size, learning rate, and iteration number, to identify the optimal settings for the iterative DPO baseline. We follow these best settings to train SELM. 
In addition, we apply iterative DPO and SELM on instruction fine-tuned models. Specifically, we consider two series of LLMs: Zephyr \citep{tunstall2023zephyr} and Llama-3 \citep{llama3}, to demonstrate the robustness of SELM. Since the official Zephyr-7B-$\beta$ model is fine-tuned with DPO on the same UltraFeedback dataset, to avoid overoptimization, we choose Zephyr-7B-SFT\footnote{https://huggingface.co/HuggingFaceH4/mistral-7b-sft-beta} as the base model and perform $3$ iterations of SELM after a single iteration of standard DPO training on the first portion of the training data (we refer to this model as Zephyr-7B-DPO). For Llama-3-8B-Instruct\footnote{https://huggingface.co/meta-llama/Meta-Llama-3-8B-Instruct} that is already fine-tuned with RLHF, we directly apply $3$ iterations of SELM training.

\subsection{Experiment Results}
\label{sec_exp_results}
We first report the performance of SELM and the baselines on the instruction-following chat benchmarks AlpacaEval 2.0 \citep{dubois2024length} and MT-Bench \citep{zheng2024judging} in Table \ref{tab_chat}. We can observe that for AlpacaEval 2.0, SELM significantly boosts Zephyr-7B-SFT and Llama-3-8B-Instruct, achieving length-controlled (LC) win rate improvements of $+16.24\%$ and $+11.75\%$, respectively. This enhancement results in models that are competitive with or even superior to much larger LLMs, such as Yi-34B-Chat \citep{young2024yi} and Llama-3-70B-Instruct. For the multi-turn MT-Bench, which exhibits higher variance, we report the average scores of SELM and DPO baselines across $3$ runs. We observe that SELM improves the scores by $+2.31$ and $+0.32$, respectively. Furthermore, the proposed method self-explores and enhances the model monotonically, with consistent performance improvements in each iteration. This validates the robustness of our algorithm. Compared to other iterative post-training algorithms, such as SPIN \citep{chen2024self}, DNO \citep{rosset2024direct}, and SPPO \citep{wu2024self}, SELM gains more improvements on both benchmarks when using the weaker base model (Zephyr-7B-SFT), and achieves the best performance when using Llama-3-8B-Instruct as the base model. 

\begin{table}[h]
\small
\centering
\begin{tabular}{l|ccc|ccc}
\hline
 & \multicolumn{3}{c|}{AlpacaEval 2.0} & \multicolumn{3}{c}{MT-Bench} \\
Model & \begin{tabular}{@{}c@{}}LC Win Rate\end{tabular} & \begin{tabular}{@{}c@{}}Win Rate \end{tabular} & \begin{tabular}{@{}c@{}}Avg. len \end{tabular} & Avgerage & \begin{tabular}{@{}c@{}}1st Turn\end{tabular} & \begin{tabular}{@{}c@{}}2nd Turn\end{tabular}\\ \hline
Zephyr-7B-SFT  & 8.01  & 4.63  & 916 & 5.30  & 5.63  & 4.97\\
Zephyr-7B-DPO  & 15.41  & 14.44  & 1752 & 7.31  &  7.55 & 7.07  \\
DPO Iter 1 (Zephyr)  &  20.53 & 16.69  & 1598 & 7.53 &  7.81  & 7.25 \\
DPO Iter 2 (Zephyr) & 22.12  & 19.82  & 1717  &  7.55 & 7.85 &  7.24 \\
DPO Iter 3 (Zephyr)  & 22.19 \textcolor{red}{\small{($\uparrow$14.18)}}  & 19.88 & 1717  & 7.46 \textcolor{red}{\small{($\uparrow$2.16)}}  & 7.85  & 7.06 \\
SELM Iter 1 (Zephyr)  & 20.52  & 17.23  & 1624 & 7.53  & 7.74 & 7.31\\
SELM Iter 2 (Zephyr) & 21.84 & 18.78  & 1665  & 7.61 & \textbf{7.85} & 7.38 \\
SELM Iter 3 (Zephyr) & \textbf{24.25}\textcolor{red}{\small{($\uparrow$16.24)}}   & \textbf{21.05}  & 1694 & \textbf{7.61} \textcolor{red}{\small{($\uparrow$2.31)}} & 7.74  & \textbf{7.49}  \\ \hline
Llama-3-8B-Instruct & 22.92  & 22.57  & 1899  & 7.93  & 8.47  & 7.38  \\
DPO Iter 1 (Llama3-It)  &  30.89 & 31.60  & 1979 & 8.07 &  8.44 & 7.70   \\
DPO Iter 2 (Llama3-It) & 33.91  & 32.95  & 1939  & 7.99  &  8.39 &  7.60 \\
DPO Iter 3 (Llama3-It)  & 33.17 \textcolor{red}{\small{($\uparrow$10.25)}} & 32.18  & 1930 & 8.18 \textcolor{red}{\small{($\uparrow$0.25)}} &  8.60 & 7.77 \\
SELM Iter 1 (Llama3-It) & 31.09  & 30.90  & 1956  & 8.09  &  8.57 & 7.61  \\
SELM Iter 2 (Llama3-It) & 33.53  & 32.61 & 1919 & 8.18  & \textbf{8.69} & 7.66   \\
SELM Iter 3 (Llama3-It) & \textbf{34.67} \textcolor{red}{\small{($\uparrow$11.75)}}  & \textbf{34.78}  & 1948 & \textbf{8.25} \textcolor{red}{\small{($\uparrow$0.32)}} & 8.53  & \textbf{7.98}  \\ \hline
SPIN &  7.23 & 6.54  & 1426 & 6.54 & 6.94 & 6.14 \\
Orca-2.5-SFT  & 10.76 & 6.99  & 1174  &  6.88 & 7.72 & 6.02  \\
DNO (Orca-2.5-SFT) & 22.59
& 24.97  & 2228 & 7.48
& 7.62  & 7.35  \\
Mistral-7B-Instruct-v0.2  & 19.39  & 15.75  & 1565 & 7.51 & 7.78  & 7.25   \\
SPPO (Mistral-it)  & 28.53
& 31.02  & 2163 & 7.59
& 7.84  & 7.34 \\ \hline
Yi-34B-Chat & 27.19  & 21.23  & 2123 & 7.90  & -  & -  \\
Llama-3-70B-Instruct & 33.17  & 33.18  & 1919 & 9.01  & 9.21  & 8.80  \\
GPT-4 Turbo (04/09) & 55.02  & 46.12  & 1802  & 9.19  & 9.38  & 9.00 \\ \hline
\end{tabular}
\caption{Results on AlpacaEval 2.0 and MT-Bench averaged with $3$ runs. Names inside the brackets are the models that are aligned based upon. The \textcolor{red}{red} arrows indicate the increment or decrement from the base model. Compared to iterative DPO and other online alignment baselines, SELM gains more improvements based on the weaker Zephyr-7B-SFT model and achieves superior performance that is competitive with much larger SOTA models when fine-tuned from Llama-3-8B-Instruct.}
\label{tab_chat}
\end{table}

Notably, the implemented iterative DPO is obtained through comprehensive grid searches of hyperparameters and practical designs (see Appendix \ref{app_exp} for details), making it a strong baseline comparable with SOTA online alignment algorithms fine-tuned from more advanced models. For example, DPO Iter 3 (Zephyr) achieves an MT-Bench score of $7.46$, representing a $2.16$ improvement over Zephyr-SFT ($5.30$) and coming close to DNO ($7.48$), which is fine-tuned from the stronger model Orca-2.5-SFT ($6.88$). Additionally, SPPO achieves an MT-Bench score of $7.59$, a modest improvement of $0.08$ over Mistral-it ($7.51$). SELM leverages the optimal hyperparameters of iterative DPO while delivering improvements with almost zero additional computational overhead.

We also conduct pairwise comparisons between SELM, iterative DPO, and the base models to validate the effectiveness of our method. The results for AlpacaEval 2.0 are shown in Figure \ref{fig_pair}. We observe that with the same number of training iterations and data, SELM consistently outperforms the iterative DPO counterpart. Additionally, when using Zephyr-7B-SFT as the base model, SELM outperforms iterative DPO even when the latter is trained with twice the data.

\begin{figure}[H]
\centering
\subfigure{
    \begin{minipage}[c]{0.47\linewidth}
        \centering
        \includegraphics[width=\textwidth]{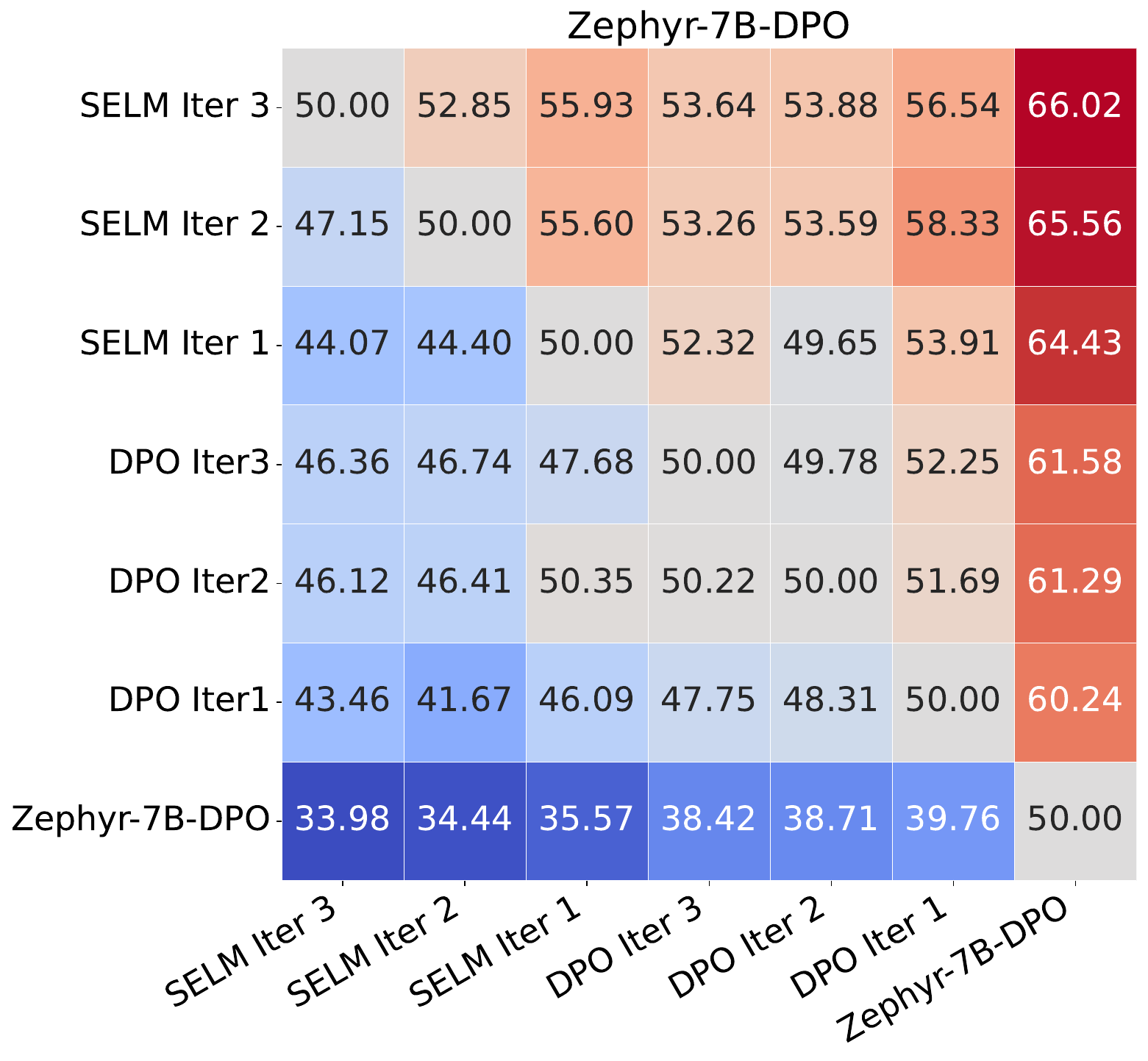}
    \end{minipage}
}
\subfigure{
    \begin{minipage}[c]{0.47\linewidth}
        \centering
        \includegraphics[width=\textwidth]{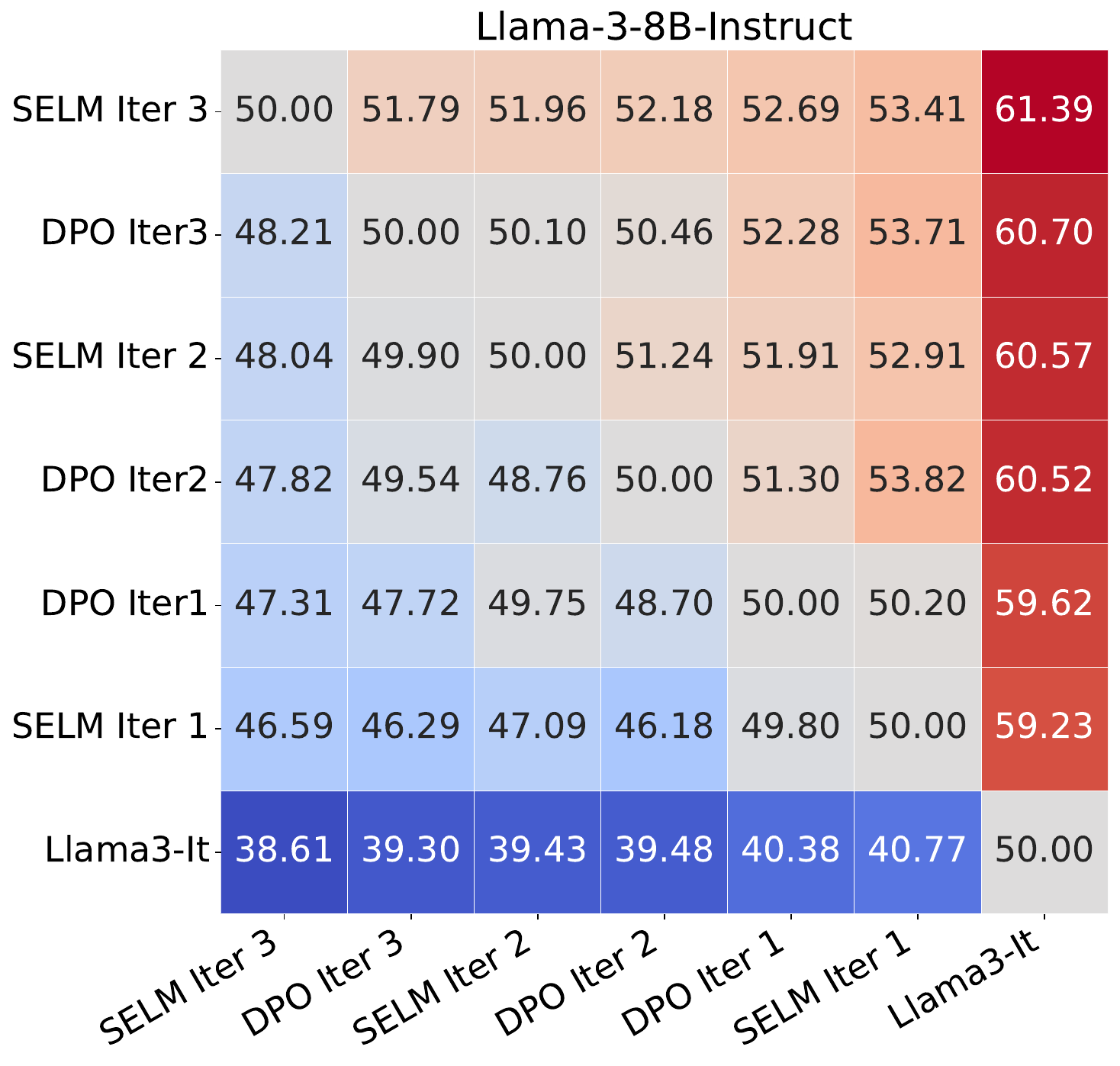}
    \end{minipage}

}\hfill
\vspace{-0.5cm}
\caption{Pairwise comparison between SELM, iterative DPO, and base models. Scores represent the LC win rates of the row models against the column models. Models positioned in higher rows have higher LC win rates against the base model and thus better performance.}
\label{fig_pair}
\end{figure}

Beyond instruction-following benchmarks, we also evaluate SELM and the baselines on several academic benchmarks, including GSM8K \citep{cobbe2021training}, HellaSwag \citep{zellers2019hellaswag}, ARC challenge \citep{clark2018think}, TruthfulQA \citep{lin2021truthfulqa}, EQ-Bench \citep{paech2023eq}, and OpenBookQA (OBQA) \citep{mihaylov2018can}. To better reflect the capabilities of LLMs, we adopt various settings for these benchmarks, including zero-shot, few-shot, and few-shot Chain-of-Thought (CoT) settings. The accuracy results for these multiple-choice QA benchmarks are provided in Table \ref{tab_academic}. It can be observed that both our method and the baselines can degrade after the RLHF phase on some benchmarks, which is known as the alignment tax \citep{askell2021general,noukhovitch2024language,li2024multi}. Nevertheless, our method is still able to improve the base models on most of the benchmarks and offers the best overall performance. 

We note that SELM is one of the instantiations of the proposed self-exploration objective in \eqref{eq_intro}, with reparameterized reward functions and algorithm-specific designs described in Section \ref{sec_algo}, such as the dataset partition and update rule. However, this objective is not restricted to the current implementation and can also be directly applied to any other online alignment framework, with or without a separate reward model, regardless of differences in algorithm designs. Thus, the proposed method is orthogonal to and can be integrated directly into the recent online RLHF workflows \citep{dong2024rlhf,xiong2023gibbs,hu2024openrlhf} that incorporate additional delicate designs with carefully curated datasets.

\vspace{-0.1cm}
\begin{table}[H]
\small
\begin{tabular}{l|ccccccc}
\hline
Models      & \begin{tabular}[c]{@{}c@{}}GSM8K\\(8-s CoT)\end{tabular} & \begin{tabular}[c]{@{}c@{}}HellaSwag\\ (10-s)\end{tabular} & \begin{tabular}[c]{@{}c@{}}ARC\\ (25-s)\end{tabular} & \begin{tabular}[c]{@{}c@{}}TruthfulQA\\ (0-s)\end{tabular} & \begin{tabular}[c]{@{}c@{}}EQ\\ (0-s)\end{tabular} & \begin{tabular}[c]{@{}c@{}}OBQA\\ (10-s)\end{tabular} & Average \\ \hline
Zephyr-7B-SFT            & 43.8 & 82.2 & 57.4 & 43.6 & 39.1 & 35.4 & 50.3 \\
Zephyr-7B-DPO     & \textcolor{red}{47.2} & 84.5 & 61.9 & 45.5 & 65.2 & 38.0 & 57.0 \\
DPO Iter 1 (Zephyr)      & 45.5 & 85.2 & 62.1 & 52.4 & 68.4 & 39.0 & 58.8 \\
DPO Iter 2 (Zephyr)      & 44.9 & 85.4 & 62.0 & 53.1 & 69.3 & 39.4 & 59.0 \\
DPO Iter 3 (Zephyr)       & 43.2 & 85.2 & 60.8 & 52.5 & 69.1 & 39.6 & 58.4 \\
SELM Iter 1 (Zephyr)       & \textcolor{blue}{46.3} & 84.8 & \textcolor{red}{62.9} & 52.9 & 68.8 & 39.6 & \textcolor{blue}{59.2} \\
SELM Iter 2 (Zephyr)      & 46.2 & \textcolor{blue}{85.4} & \textcolor{blue}{62.1} & \textcolor{blue}{53.1} & \textcolor{blue}{69.3} & \textcolor{blue}{39.6} & \textcolor{red}{59.3} \\
SELM Iter 3 (Zephyr)   & 43.8 & \textcolor{red}{85.4} & 61.9 & 52.4 & \textcolor{red}{69.9} & \textcolor{red}{39.8} & 58.9 \\ \hline
Llama-3-8B-Instruct  & 76.7 & 78.6 & 60.8 & 51.7 & 61.8 & 38.0 & 61.3 \\
DPO Iter 1 (Llama3-It)      & 78.5 & 81.7 & 63.9 & 55.5 & 64.1 & 42.6 & 64.4 \\
DPO Iter 2 (Llama3-It)      & 79.4 & 81.7 & 64.4 & 56.4 & \textcolor{red}{64.3} & 42.6 & 64.8 \\
DPO Iter 3 (Llama3-It)       & \textcolor{blue}{80.1} & 81.7 & 64.1 & 56.5 & 64.1 & 42.6 & 64.8 \\
SELM Iter 1 (Llama3-It)      & 78.7 & 81.7 & \textcolor{red}{64.5} & 55.4 & 64.1 & 42.4 & 64.5 \\
SELM Iter 2 (Llama3-It)      & 79.3 & \textcolor{blue}{81.8} & \textcolor{blue}{64.7} & \textcolor{blue}{56.5} & \textcolor{blue}{64.2} & \textcolor{blue}{42.6} & \textcolor{blue}{64.9} \\
SELM Iter 3 (Llama3-It)     & \textcolor{red}{80.1} & \textcolor{red}{81.8} & 64.3 & \textcolor{red}{56.5} & 64.2 & \textcolor{red}{42.8} & \textcolor{red}{65.0} \\ \hline
SPIN            & 44.7 & 85.9 & 65.9 & 55.6 & 54.4 & 39.6 & 57.7 \\
Mistral-7B-Instruct-v0.2   & 43.4 & 85.3 & 63.4 & 67.5 & 65.9 & 41.2 & 61.1 \\
SPPO (Mistral-it)    & 42.4 & 85.6 & 65.4 & 70.7 & 56.5 & 40.0 & 60.1 \\ \hline
\end{tabular}
\vspace{-0.1cm}
\caption{Performance comparison between SELM and the baselines on academic multi-choice QA benchmarks in standard zero-shot, few-shot, and CoT settings. Here, n-s refers to n-shot. The \textcolor{red}{red} and \textcolor{blue}{blue} texts represent the best and the second-best results. \vspace{-0.2cm}}
\label{tab_academic}
\end{table}

\subsection{Ablation Study}
\begin{figure}[h]
    \centering
    \begin{minipage}[b]{0.32\linewidth}
        \centering
        \includegraphics[width=\linewidth]{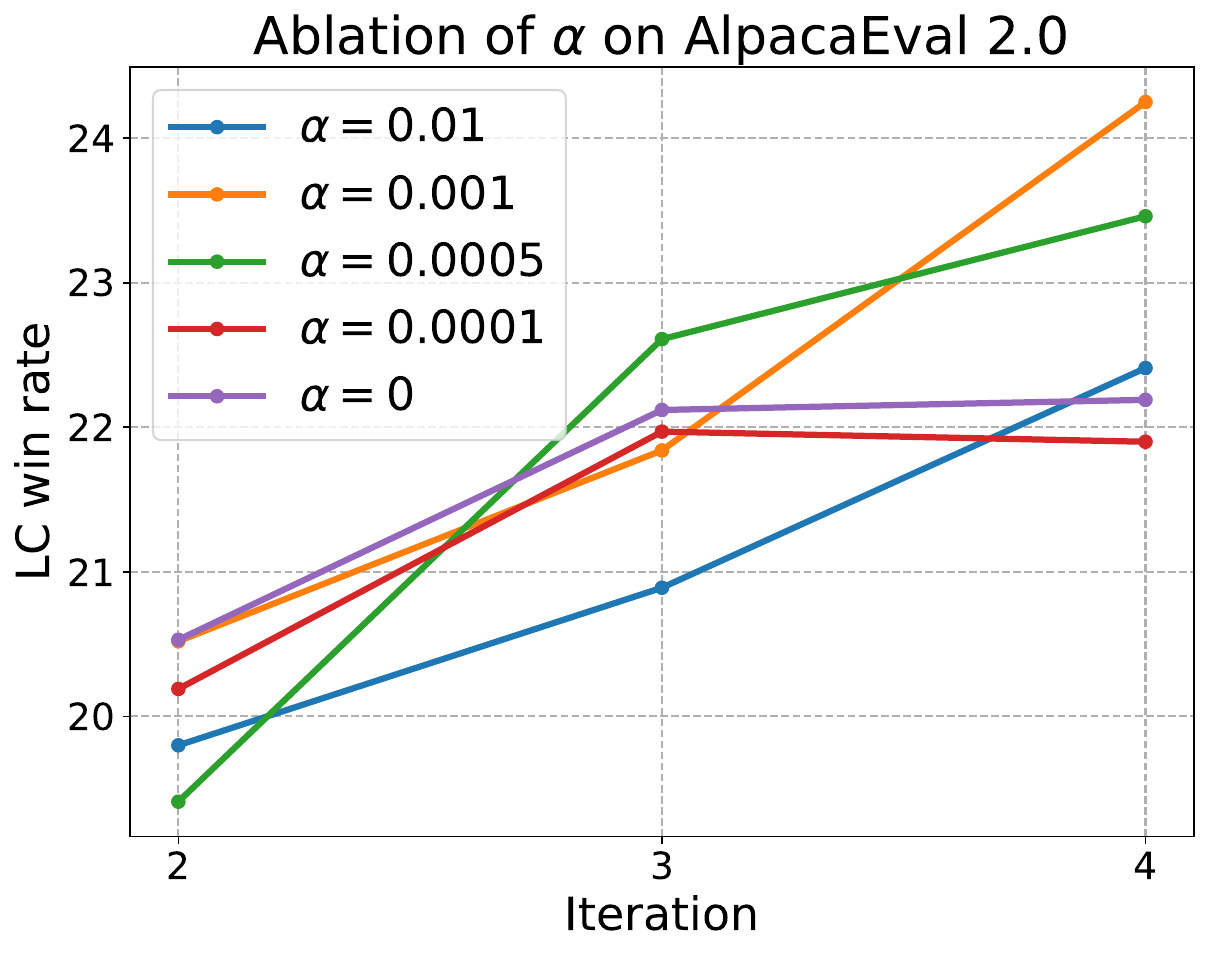}
    \end{minipage}
    \begin{minipage}[b]{0.32\linewidth}
        \centering
        \includegraphics[width=\linewidth]{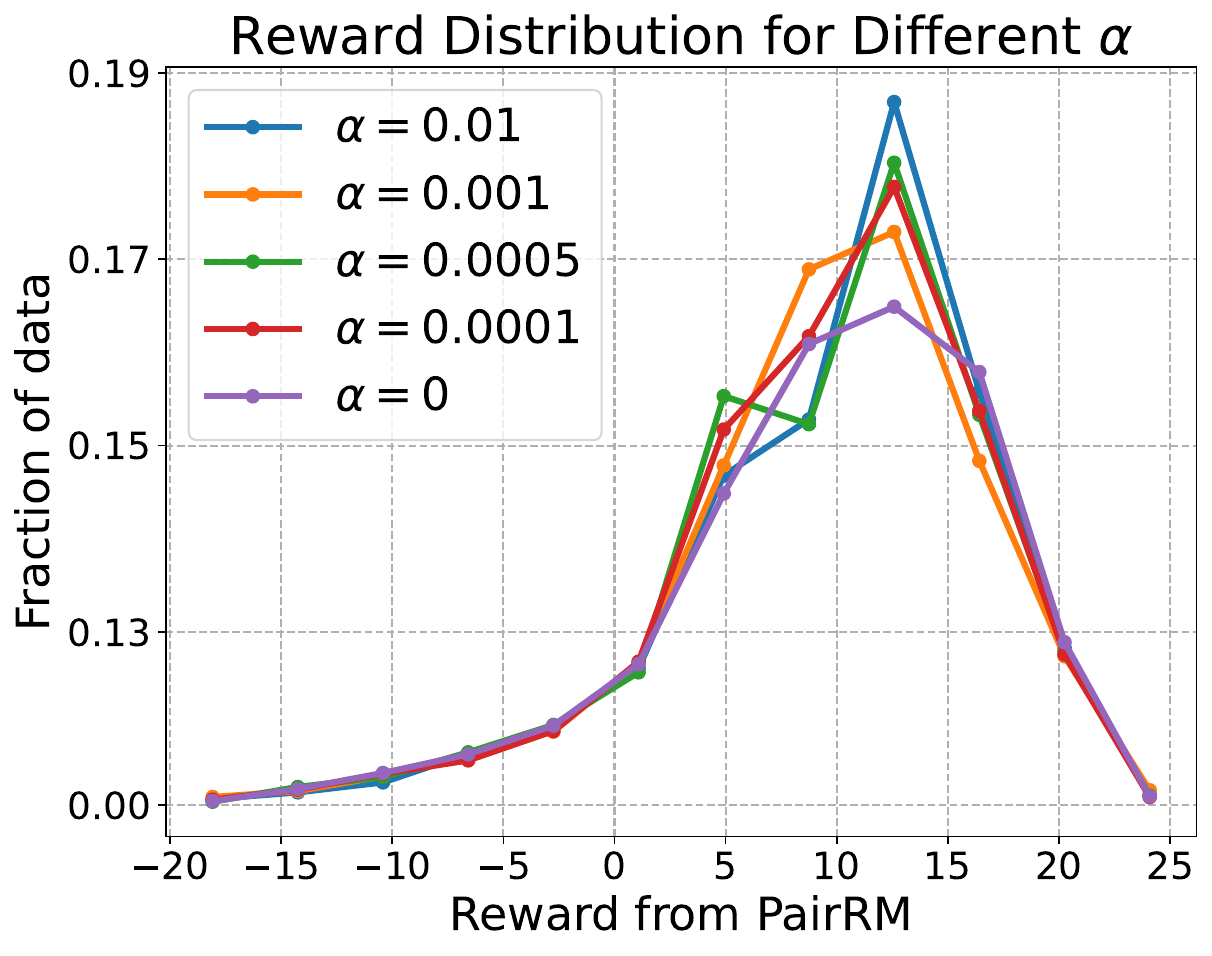}
    \end{minipage}
    \begin{minipage}[b]{0.32\linewidth}
        \centering
        \includegraphics[width=\linewidth]{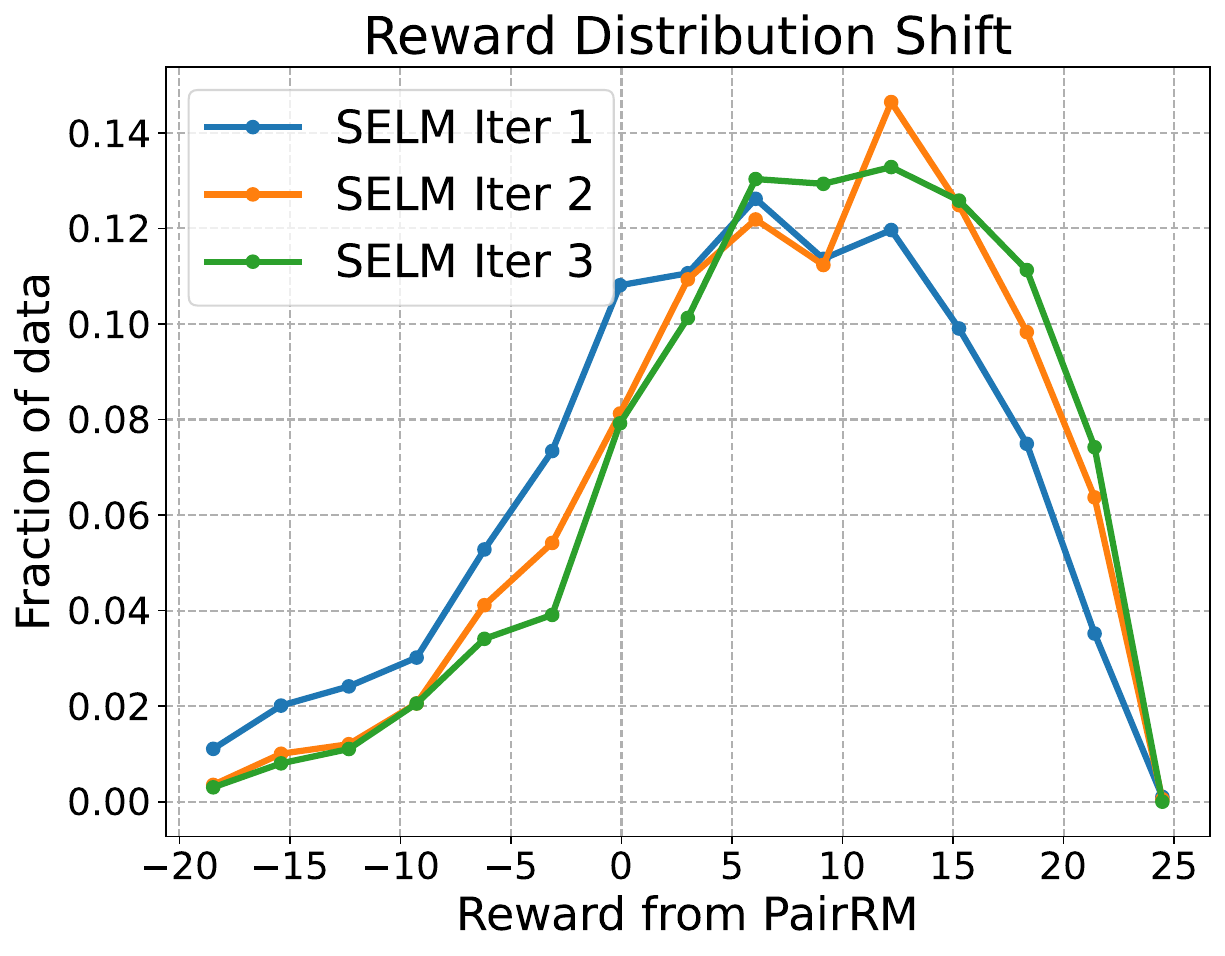}
    \end{minipage}
\vspace{-0.1cm}
\caption{Ablation on the optimism coefficient $\alpha$ and the change of the reward distribution. \textbf{Left:} The length-controlled win rates of SELM with different $\alpha$ on AlpacaEval 2.0. \textbf{Middle:} Comparison of reward distributions at iteration 2 with different $\alpha$. \textbf{Right:} SELM initially explores and then shifts to higher-reward regions as more training iterations are performed.}
\label{fig_alpha}
\end{figure}

We first provide ablation studies to better understand the explorative optimism term.
We begin by investigating the effect of the optimism coefficient $\alpha$. In Figure \ref{fig_alpha} (Left), we plot the LC win rates of SELM when using Zephyr-7B-SFT as the base model for different $\alpha$ in the AlpacaEval 2.0 benchmark. We find that setting a small $\alpha$, such as $0.0001$, leads to very similar behaviors to the iterative DPO ($\alpha=0$) baseline, while SELM with a large $\alpha$ may become overly optimistic and thus not very effective. These results meet our expectations, suggesting that proper values of $\alpha$ are essential for achieving the best trade-off between exploration and exploitation.

Next, we study the difference in reward distributions with varied $\alpha$ and iterations. Specifically, for prompts from the 2k test set of UltraFeedback, we greedily sample from the LLM and generate rewards for the responses with PairRM. We then calculate the fraction of data that lies in each partition of rewards. The results for different $\alpha$ values of SELM Iter 2 (Zephyr) in Figure \ref{fig_alpha} (Middle) indicates that increasing $\alpha$ results in distributions that are concentrated in higher-reward regions. 

\begin{figure}[h!]
    \centering
    \begin{minipage}[t]{0.48\textwidth}
        \centering
        \includegraphics[width=\linewidth]{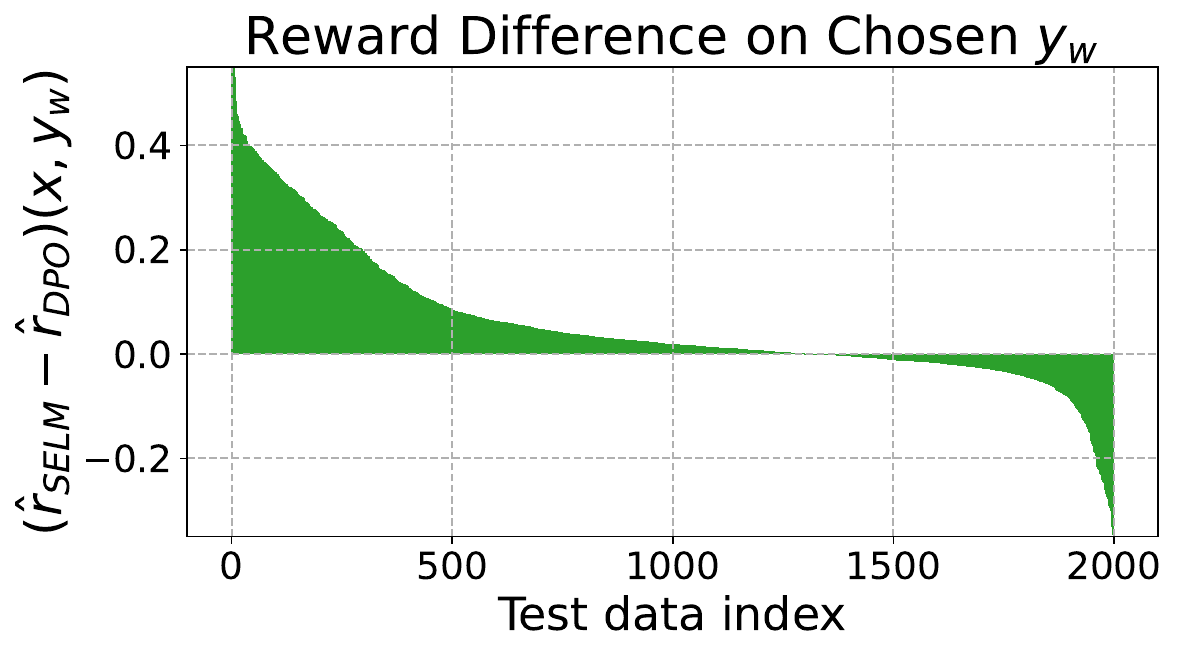}
        \label{fig_ir_chosen}
    \end{minipage}
    \hfill
    \begin{minipage}[t]{0.48\textwidth}
        \centering
        \includegraphics[width=\linewidth]{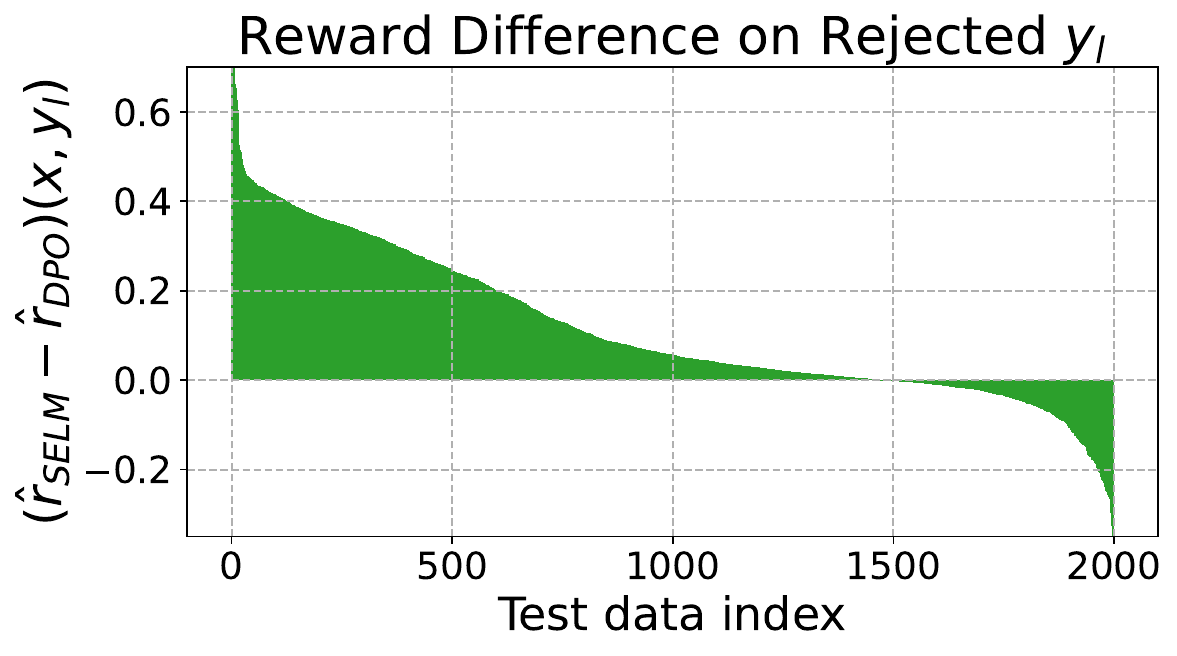}
        \label{fig_ir_rejected}
    \end{minipage}
    \vspace{-0.4cm}
    \caption{Difference of implicit reward between SELM and DPO on the chosen and rejected responses. SELM assigns a higher implicit reward than DPO for both responses.}
    \label{fig_ir}
\end{figure}

Additionally, Figure \ref{fig_alpha} (Right) demonstrates that the reward distribution shifts to the right (higher) as more training iterations are performed. This shift corresponds to an initial exploration phase, where the LLM generates uncertain responses of varying quality, followed by an exploitation phase as feedback is incorporated and more training data is collected.

We also conduct ablation studies on the implicit reward captured by the SELM and DPO models. Recall that for both SELM and DPO, the implicit reward takes the form of $\hat{r}_\theta(x, y) = \beta(\log\pi_\theta(y\mid x) - \log\pi_{\text{ref}}(y\mid x))$. We calculate the reward difference $\hat{r}_{\text{SELM}}(x, y) - \hat{r}_{\text{DPO}}(x, y)$ for each prompt $x$ in the UltraFeedback holdout test set. Here, we study the implicit reward of the good (chosen) and bad (rejected) responses, so $y=y_w$ or $y=y_l$. We then sort the reward difference and plot the results for Zephyr-based models after iteration 1 in Figure \ref{fig_ir}. The plot clearly shows that for both chosen and rejected responses, SELM produces higher \textit{implicit} rewards compared to DPO, aligning with the proposed optimistically biased self-exploration objective.

\begin{wrapfigure}{r}{0.48\textwidth}
    \begin{minipage}[t]{0.48\textwidth}
        \centering
        \begin{minipage}[t]{\linewidth}
            \centering
            \includegraphics[width=\linewidth]{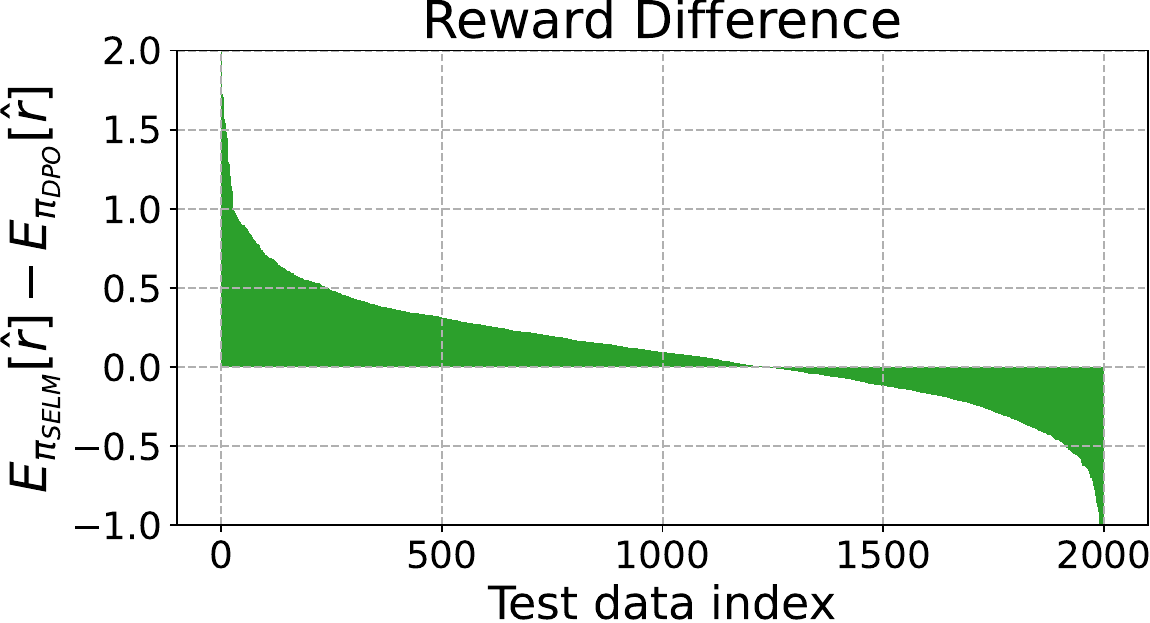}
        \end{minipage}
        \vspace{-0.65cm}
        \setlength{\belowcaptionskip}{-15pt}
        \caption{SELM actively explores by favoring high-reward responses.}
        \label{fig_explore}
    \end{minipage}
\end{wrapfigure}
In Section \ref{analysis}, we show that SELM engages in more active exploration by prioritizing high-reward responses compared to DPO, which indiscriminately favors unseen extrapolations and explores passively. To validate this, we sample three responses from SELM and DPO Iter 2 (Zephyr) for each prompt and we calculate the subtraction of the mean implicit rewards. As illustrated in Figure \ref{fig_explore}, SELM consistently achieves higher implicit rewards across most prompts, with the positive reward differences being notably larger in magnitude, supporting our claim regarding SELM’s active exploration behavior.
\section{Conclusion \& Future Work}
\label{conclusion}
In this paper, we introduced an active preference elicitation method for the online alignment of large language models. By incorporating an optimism term into the reward-fitting objective, the proposed bilevel self-exploring objective effectively balances between exploiting observed data and exploring potentially high-reward regions. Unlike standard online RLHF algorithms that passively explore the response space by sampling from the training LLM, whose sole objective is maximizing the expected learned reward, our method actively seeks diverse and high-quality responses. This self-exploration mechanism helps mitigate the risk of premature convergence and overfitting when the reward model is only locally accurate. To optimize this bilevel objective, we solve the inner-level problem and reparameterize the reward with the LLM policy, resulting in a simple yet novel iterative alignment algorithm called \textit{Self-Exploring Language Models} (SELM). Compared to DPO, SELM is provably sample-efficient and improves the exploration efficiency by selectively favoring responses with high potential rewards rather than indiscriminately sampling unseen responses.

Our experiments, conducted with Zephyr-7B-SFT and Llama-3-8B-Instruct models, demonstrate the efficacy of SELM with consistent improvements on AlpacaEval 2.0, MT-Bench, and academic benchmarks with minimal computational overhead. These results underscore the ability of SELM to enhance the alignment and capabilities of LLMs by promoting more diverse and high-quality responses. Since the proposed technique is orthogonal to the adopted online RLHF workflow, it will be interesting to apply our method within more sophisticated alignment frameworks with advanced designs, which we would like to leave as future work.

\bibliographystyle{ims}
\bibliography{reference.bib}

\appendix
\section{Derivations in Section \ref{sec_der}}
\label{derivation_1}
We begin by deriving \eqref{final_obj}. The solution for the inner-level optimization problem of \eqref{our_obj} is as follows:
\#\label{eq_app}
\max_\pi \mathcal{F}(\pi; r) &= \max_\pi\EE_{\substack{x\sim\mathcal{D}_t, y\sim\pi(\cdot\mid x)\\ y'\sim\pi_{\text{ref}}(\cdot\mid x)}}\Bigl[r(x, y) - r(x, y')\Bigr] - \beta\mathbb{D}_{\text{KL}}(\pi \,|| \,\pi_{\text{ref}})\notag\\
&= \EE_{x\sim\mathcal{D}_t}\Bigl[\beta\log\EE_{y\sim\pi_{\text{ref}}(\cdot\mid x)}\bigl[\exp(r(x, y) / \beta)\bigr]\Bigr] - \EE_{x\sim\mathcal{D}_t, y'\sim\pi_{\text{ref}}(\cdot\mid x)}\bigl[r(x, y')\bigr]
\#

When the reward $r$ is reparameterized by $\hat{r}_\theta (x, y) = \beta(\log\pi_\theta(y\mid x) - \log\pi_{\text{ref}}(y\mid x))$, we have that the first term in \eqref{eq_app} is $0$. The bilevel objective \eqref{our_obj} then becomes
\$
\max_r  -\mathcal{L}_{\text{lr}}(r; \mathcal{D}_t) - \alpha\EE_{x\sim\mathcal{D}, y'\sim\pi_{\text{ref}}(\cdot\mid x)}\bigl[r(x, y')\bigr].
\$

By reparameterizing the reward with the LLM, we obtain the desired results in \eqref{final_obj}.

Then we provide the derivation of \eqref{eq_grad}. We primarily consider the gradient of the newly incorporated term $\EE_{x\sim\mathcal{D}, y\sim\pi_{\text{ref}}(\cdot\mid x)}[\log\pi_\theta(y\mid x)]$. Specifically, we have
\$
\nabla_\theta\EE_{x\sim\mathcal{D}, y\sim\pi_{\text{ref}}(\cdot\mid x)}\bigl[\log\pi_\theta(y\mid x)\bigr] &= \EE_{x\sim\mathcal{D}}\Bigl[\sum_y\pi_{\text{ref}}(y\mid x)\nabla_\theta\log\pi_\theta(y\mid x)\Bigr]\\
&= \EE_{x\sim\mathcal{D}, y\sim\pi_\theta}\Bigl[\frac{\pi_{\text{ref}}(y\mid x)}{\pi_\theta(y\mid x)}\nabla_\theta\log\pi_\theta(y\mid x)\Bigr]\\
&= \EE_{x\sim\mathcal{D}, y\sim\pi_\theta}\Bigl[\exp\bigl(-\hat{r}_\theta(x. y) / \beta\bigr)\nabla_\theta\log\pi_\theta(y\mid x)\Bigr].
\$

For the derivation of the DPO gradient $\nabla_\theta\mathcal{L}_{\text{DPO}}(\pi_\theta; \mathcal{D}_t)$, we refer the readers to \cite{rafailov2024direct}.

\section{Proof of Theorem \ref{thm}}
\label{proof_indis}
\begin{proof}[Proof of Theorem \ref{thm}]
The solution to the KL-constrained reward minimization objective \eqref{eq_pi_rho} is
\$
\pi^{\min}_\rho(y\mid x) = \pi_{\rho}(y\mid x)\exp\bigl(-\hat{r}_\rho(x, y) / \beta\bigr) / Z(x),
\$
where $Z(x) = \sum_y\pi_\rho(y\mid x)\exp(-\hat{r}_\rho(x, y)/\beta) = 1$. Then we have $\pi^{\min}_\rho(y\mid x) = \pi_{\text{ref}}(y\mid x)$,
i.e., the reference policy $\pi_{\text{ref}}$ achieves the lowest implicit reward reparameterized by any $\rho$.
\end{proof}

\section{Proof of Theorem \ref{thm:regret}}
We use the reduction technique from \cite{xie2024exploratory} to connect the sample complexity of SELM to that of existing RL algorithms \citep{zhong2022gec, liu2024maximize}. We restate the proof technique from \cite{xie2024exploratory} for completeness. We emphasize that it is not a novel contribution of the present work. It is worth noting that the theoretical version of the self-exploration mechanism (Algorithm~\ref{alg_se_theory}) is a bit different from the practical one used in the numerical experiments and is closer to the proposed algorithm in \cite{xie2024exploratory}.

We present the following theoretical version of the proposed self-exploration algorithm. The key modification in Algorithm \ref{alg_se} lies in its pragmatic strategy for constructing the chosen and rejected responses. Despite this adjustment, the core principles of leveraging the self-exploration objective during online alignment remain the same.
\label{proof_thm}
\begin{algorithm}[H]
\caption{Self-Exploring Language Models (SELM; Theoretical Version)}
\begin{algorithmic}[1]\label{alg_se_theory}
\REQUIRE Reference model $\pi_{\text{ref}}$, preference dataset $\mathcal{D}_0 = \emptyset$, prompt distribution $\nu$, online iterations $T$, optimism coefficient $\alpha$, $\pi_0 = \pi_{\text{ref}}$.
\FOR{iteration $t = 1, 2, \ldots, T$}
\STATE Sample $x_t \sim \nu$, $y_t^1 \sim \pi_{t-1}(\cdot \mid x)$, $y_t^2 \sim \pi_{\mathrm{ref}}(\cdot \mid x)$.
\STATE Update the preference data $\cD_t = \cD_{t-1} \cup \{(x_t, y_t^1, y_t^2)\}$
\STATE Train the LLM $\pi_{t} = \argmax_{\pi} \{ -\mathcal{L}_{\text{DPO}}(\pi; \mathcal{D}_t) - \alpha \cdot \EE_{x\sim\nu} \EE_{y \sim \pi_{\mathrm{ref}}(\cdot \mid x)}[\log\pi(y \mid x)]\}$, let $\pi_{\text{ref}}=\pi_{t}$.
\ENDFOR 
\end{algorithmic}
\end{algorithm}

\begin{definition}[Preference-based GEC] \label{def:pgec}
    For the function class $\Pi$, we define the preference-based GEC (PGEC) as the smallest $d_{\mathrm{GPEC}}$ as
    \small
    \$
    & \sum_{t = 1}^T \EE_{(x, y, y') \sim (\nu, \pi_{\mathrm{ref}}, \pi_t)} \left[  \log \frac{\pi^*(y \given x)}{\pi_{\mathrm{ref}}(y \given x)} - \log \frac{\pi_t(y \given x)}{\pi_{\mathrm{ref}}(y \given x)} - \log \frac{\pi^*(y' \given x)}{\pi_{\mathrm{ref}}(y' \given x)} +  \log \frac{\pi_t(y' \mid x)}{\pi_{\mathrm{ref}}(y' \mid x)} \right] \notag \\
& \le \sqrt{d_{\mathrm{PGEC}} \sum_{t = 1}^T \sum_{\tau = 1}^{t - 1} \EE_{(x, y, y') \sim (\nu, \pi_{\mathrm{ref}}, \pi^\tau)} \left[  \log \frac{\pi^*(y \given x)}{\pi_{\mathrm{ref}}(y \given x)} - \log \frac{\pi^\tau(y \given x)}{\pi_{\mathrm{ref}}(y \given x)} - \log \frac{\pi^*(y' \given x)}{\pi_{\mathrm{ref}}(y' \given x)} +  \log \frac{\pi^\tau(y' \mid x)}{\pi_{\mathrm{ref}}(y' \mid x)} \right]^2 }  \\
& \qquad + 4\sqrt{ d_{\mathrm{PGEC}} T} .
    \$
\end{definition}

The definition of PGEC is a preference-based version of Generalized Eluder Coefficient (GEC) proposed by \citep{zhong2022gec}. Intuitively, both PGEC and GEC establish a crucial connection between \emph{prediction error} and \emph{in-sample estimation error}, effectively transforming regret minimization into an online estimation problem. For a comprehensive explanation and in-depth discussion, readers are directed to \citet{zhong2022gec}. A slight difference is that the PGEC here is defined with respect to the policy class, while the GEC in \citet{zhong2022gec} is defined with respect to the model or value class. These can be connected if we regard the implicit reward class $\log (\pi/\pi_{\mathrm{ref}})$ as the model or value class. As an important example, if we consider the log-linear function class $\Pi = \{\pi_\theta: \pi_\theta(y \mid x) \propto \exp( \langle \phi(x, y), \theta \rangle /\beta )\}$, we can show that $d_{\mathrm{PGEC}} = \tilde{O}(d)$ by the elliptical potential lemma \citep{abbasi2011improved,zhong2022gec}. Another remark is that here the PGEC is defined in the bandit formulation, and it can be naturally extended to the token-wise MDP formulation \citep{zhong2024dpo,rafailov2024r,xie2024exploratory} and further connects to the eluder dimension in the context of preference-based MDPs \citep{chen2022human,wang2023rlhf}. Specifically, if we regard the generation process of LLMs as token-level MDPs where the generation of each token serves as one step, the learning objective is maximizing
\$
\cJ(\pi) = \EE_{x \sim \nu, \tau \sim \pi} \left[ r(\tau) - \beta \log \frac{\pi(\tau \given x)}{\pi_{\mathrm{ref}}(\tau \given x)} \right].
\$
Here $\tau$ is the full trajectory starting from $x$. We can similarly define the PGEC (Definition~\ref{def:pgec}) for token-wise MDPs by replacing the response $y, y'$ in the bandit formulation with the trajectories $\tau, \tau'$ in the token-wise MDP formulation. We have the following informal theorem:
\begin{theorem}[Regret for MDP Formulation (informal)]
    With proper parameter choice, it holds with probability at least $1 - \delta$ that
    \$
    \mathcal{R}(T) &\lesssim \sqrt{d_{\mathrm{PGEC}} \cdot \exp(2V_{\max}) \cdot T \cdot \log (|\Pi|/\delta)},
    \$
    where $V_{\max}$ is a bounded coefficient for toekn-wise MDPs, similar to the one described in Assumption~\ref{assumption:policy:class}.
\end{theorem}

\subsection{Proof of Theorem \ref{thm:regret}}

\begin{proof}[Proof of Theorem \ref{thm:regret}]
We first decompose the regret as
{\small
\$
\mathcal{R}(T) & = \sum_{t = 1}^T [\cJ(\pi^*) - \cJ(\pi_t)] \\
& = \sum_{t = 1}^T \left( \EE_{x \sim \nu, y \sim \pi^*(\cdot \mid x)}\left[r(x, y) - \beta \log \frac{\pi^*(y \mid x)}{\pi_{\mathrm{ref}}(y \mid x)}\right] - \EE_{x \sim \nu, y \sim \pi_t(\cdot \mid x)}\left[r(x, y) - \beta \log \frac{\pi_t(y \mid x)}{\pi_{\mathrm{ref}}(y \mid x)}\right] \right)\\
& = \sum_{t = 1}^T \left( \EE_{x \sim \nu, y \sim \pi_{\mathrm{ref}}(\cdot \mid x)}\left[r(x, y) - \beta \log \frac{\pi^*(y \mid x)}{\pi_{\mathrm{ref}}(y \mid x)}\right] - \EE_{x \sim \nu, y \sim \pi_t(\cdot \mid x)}\left[r(x, y) - \beta \log \frac{\pi_t(y \mid x)}{\pi_{\mathrm{ref}}(y \mid x)}\right] \right),
\$}
where the last line uses the fact that 
\# \label{eq:closed:form}
r(x, y) - \beta \log \frac{\pi^*(y \mid x)}{\pi_{\mathrm{ref}}(y \mid x)} = \beta \cdot \log Z_r(x), 
\#
which is independent of the response $y$. Rearranging the above regret decomposition, we have
{\small
\# \label{eq:220}
\mathcal{R}(T) & = \sum_{t = 1}^T \left( \EE_{x \sim \nu, y \sim \pi_{\mathrm{ref}}(\cdot \mid x)}\left[r(x, y) - \beta \log \frac{\pi^*(y \mid x)}{\pi_{\mathrm{ref}}(y \mid x)}\right] - \EE_{x \sim \nu, y \sim \pi_t(\cdot \mid x)}\left[r(x, y) - \beta \log \frac{\pi_t(y \mid x)}{\pi_{\mathrm{ref}}(y \mid x)}\right] \right) \notag \\
& = \sum_{t = 1}^T \EE_{x \sim \nu, y \sim \pi_{\mathrm{ref}}(\cdot \mid x)} \left[ \beta \log \frac{\pi_t(y \mid x)}{\pi^*(y \mid x)} \right]    \notag \\
& \qquad + \sum_{t = 1}^T \EE_{x \sim \nu, y \sim \pi_{\mathrm{ref}}(\cdot \mid x), y' \sim \pi_t(\cdot \mid x)} \left[ r(x, y) - \beta \log \frac{\pi_t(y \given x)}{\pi_{\mathrm{ref}}(y \given x)} - r(x, y') + \beta \log \frac{\pi_t(y' \mid x)}{\pi_{\mathrm{ref}}(y' \mid x)} \right] \notag \\
& = \sum_{t = 1}^T \EE_{x \sim \nu, y \sim \pi_{\mathrm{ref}}(\cdot \mid x)} \left[ \beta \log \frac{\pi_t(y \mid x)}{\pi^*(y \mid x)} \right]    \notag \\
& \qquad + \beta \sum_{t = 1}^T \EE_{(x, y, y') \sim (\nu, \pi_{\mathrm{ref}}, \pi_t)} \left[  \log \frac{\pi^*(y \given x)}{\pi_{\mathrm{ref}}(y \given x)} - \log \frac{\pi_t(y \given x)}{\pi_{\mathrm{ref}}(y \given x)} - \log \frac{\pi^*(y' \given x)}{\pi_{\mathrm{ref}}(y' \given x)} +  \log \frac{\pi_t(y' \mid x)}{\pi_{\mathrm{ref}}(y' \mid x)} \right], 
\#}
where the last line uses \eqref{eq:closed:form}.
By the definition of PGEC in Definition~\ref{def:pgec}, we have
{
\small
\# \label{eq:221}
& \sum_{t = 1}^T \EE_{(x, y, y') \sim (\nu, \pi_{\mathrm{ref}}, \pi_t)} \left[  \log \frac{\pi^*(y \given x)}{\pi_{\mathrm{ref}}(y \given x)} - \log \frac{\pi_t(y \given x)}{\pi_{\mathrm{ref}}(y \given x)} - \log \frac{\pi^*(y' \given x)}{\pi_{\mathrm{ref}}(y' \given x)} +  \log \frac{\pi_t(y' \mid x)}{\pi_{\mathrm{ref}}(y' \mid x)} \right] \notag \\
& \le \sqrt{d_{\mathrm{PGEC}} \sum_{t = 1}^T \sum_{\tau = 1}^{t - 1} \EE_{(x, y, y') \sim (\nu, \pi_{\mathrm{ref}}, \pi^\tau)} \left[  \log \frac{\pi^*(y \given x)}{\pi_{\mathrm{ref}}(y \given x)} - \log \frac{\pi^\tau(y \given x)}{\pi_{\mathrm{ref}}(y \given x)} - \log \frac{\pi^*(y' \given x)}{\pi_{\mathrm{ref}}(y' \given x)} +  \log \frac{\pi^\tau(y' \mid x)}{\pi_{\mathrm{ref}}(y' \mid x)} \right]^2 } \notag \\ 
& \qquad + 4 \sqrt{d_{\mathrm{PGEC}} T} \notag \\
& \le \frac{d_{\mathrm{PGEC}}}{4\eta} + \eta \sum_{t = 1}^T \sum_{\tau = 1}^{t - 1} \EE_{(x, y, y') \sim (\nu, \pi_{\mathrm{ref}}, \pi^\tau)} \left[  \log \frac{\pi^*(y \given x)}{\pi_{\mathrm{ref}}(y \given x)} - \log \frac{\pi^\tau(y \given x)}{\pi_{\mathrm{ref}}(y \given x)} - \log \frac{\pi^*(y' \given x)}{\pi_{\mathrm{ref}}(y' \given x)} +  \log \frac{\pi^\tau(y' \mid x)}{\pi_{\mathrm{ref}}(y' \mid x)} \right]^2, \notag \\
& \qquad + 4 \sqrt{d_{\mathrm{PGEC}} T},
\#}
where the last inequality follows from the fact that $\sqrt{xy} \le x/(4\eta) + \eta y$ for any $x, y, \eta > 0$.

By the updating rule of $\pi_{t + 1} = \argmax_{\pi} \{ -\mathcal{L}_{\text{DPO}}(\pi; \mathcal{D}_t) - \alpha \cdot \EE_{x\sim\nu} \EE_{y \sim \pi_{\mathrm{ref}}(\cdot \mid x)}[\log\pi(y \mid x)]\}$, we have
\$
 & -\mathcal{L}_{\text{DPO}}(\pi_{t}; \mathcal{D}_{t-1}) - \alpha \cdot \EE_{x\sim\nu, y \sim \pi_{\mathrm{ref}}(\cdot \mid x)} [\log\pi_{t}(y \mid x)] \\
 & \qquad \ge  -\mathcal{L}_{\text{DPO}}(\pi^*; \mathcal{D}_{t-1}) - \alpha \cdot \EE_{x\sim\nu, y \sim \pi_{\mathrm{ref}}(\cdot \mid x)}[\log\pi^*(y \mid x)],
\$
which equivalents to that
\# \label{eq:222}
\EE_{x\sim\nu, y \sim \pi_{\mathrm{ref}}(\cdot \mid x)} \left[ \beta \log \frac{\pi_t(y \mid x)}{\pi^*(y \mid x)} \right] \le \frac{\beta}{\alpha} \cdot \left( \cL_{\mathrm{DPO}}(\pi^*; \cD_{t-1}) - \cL_{\mathrm{DPO}}(\pi_t; \cD_{t-1}) \right). 
\#
We upper bound the right handsise of \eqref{eq:222} via the following lemma.
\begin{lemma}[Concentration] \label{lemma:concentration}
For any $t \in [T]$ and $0 < \delta < 1$, it holds with probability $1- \delta$ that
{\small
\$
& \cL_{\mathrm{DPO}}(\pi^*; \cD_{t-1}) - \cL_{\mathrm{DPO}}(\pi_t; \cD_{t-1}) \\
& \lesssim -\frac{2}{\exp(4R_{\max})} \cdot \sum_{\tau = 1}^{t - 1} \EE_{(x, y, y') \sim (\nu, \pi_{\mathrm{ref}}, \pi^\tau)} \left[  \log \frac{\pi^*(y \given x)}{\pi_{\mathrm{ref}}(y \given x)} - \log \frac{\pi^\tau(y \given x)}{\pi_{\mathrm{ref}}(y \given x)} - \log \frac{\pi^*(y' \given x)}{\pi_{\mathrm{ref}}(y' \given x)} +  \log \frac{\pi^\tau(y' \mid x)}{\pi_{\mathrm{ref}}(y' \mid x)} \right]^2 \\
&  \qquad + \log(|\Pi|/\delta).
\$}
\end{lemma}

\begin{proof}
    The proof of this lemma follows the standard MLE analysis \citep{zhang2006varepsilon} and its application for standard reward-based RL \citep{agarwal2020flambe,liu2024maximize}. Recent works \citep{liu2024provably,xie2024exploratory,cen2024value} also applies this result for RLHF. For brevity, we omit the detailed proof here and direct readers to these related works for the proof.
\end{proof}

Combining \eqref{eq:220}, \eqref{eq:221}, \eqref{eq:222}, and Lemma~\ref{lemma:concentration}, together with the parameter choice $\alpha = 2/(\eta \exp(4R_{\max}))$, we obtain
\$
\cR(T) &\lesssim \frac{\beta T d_{\mathrm{PGEC}}}{\eta} + \beta \eta \cdot {\exp(4R_{\max})} \log (|\Pi|/\delta) + 4 \sqrt{d_{\mathrm{PGEC}} T} \\
& \lesssim \sqrt{d_{\mathrm{PGEC}} \cdot \exp(2R_{\max}) \cdot T \cdot \log (|\Pi|/\delta)},
\$
where the last line follows from the fact that $\eta = \sqrt{T d_{\mathrm{PGEC}}/(\exp(4R_{\max}) \log(|\Pi|/\delta)) }$. Therefore, we finish the proof of Theorem~\ref{thm:regret}.
\end{proof}

\section{Experiment Setup}
\label{app_exp}
In experiments, we use the Alignment Handbook \citep{alignment_handbook2023} framework as our codebase. We find the best hyperparameter settings for the strong iterative DPO baseline by conducting a grid search over the iteration number, batch size, learning rate, and label update rule. The results for the Zephyr-based models are shown in Figure \ref{fig_grid}. Specifically, we find that using the same amount of data, updating the model too many iterations can lead to instability. So we set the iteration number to $3$ for Llama3-It-based and Zephyr-based models (excluding the first iteration of DPO training). Besides, we observe that choosing different batch sizes has a large effect on the models' performance and the optimal batch size heavily depends on the model architecture. In experiments, we set the batch size to $256$ and $128$ for the Zephyr-based and Llama3-It-based models, respectively. For the learning rate, we consider three design choices: cyclic learning rate with constant cycle amplitude, linearly decayed cycle amplitude, and decayed cycle amplitude at the last iteration. We find that a decaying cycle amplitude performs better than constant amplitudes in general. Thus, for Zephyr-based models, we set the learning to $5e-7$ for the first three iterations and $1e-7$ for the last iteration. In each iteration, the warmup ratio is $0.1$. For Llama3-It-based models, we use a linearly decayed learning rate from $5e-7$ to $1e-7$ within $3$ iterations with the same warmup ratio. We also test two update ways for the preference data. One is to rank $y_w, y_l, y_{\text{ref}}$ and keep the best and worst responses in the updated dataset, which is the setting that is described in the main paper. The other is to compare $y_w$ and $y_{\text{ref}}$ and replace the chosen or rejected response by $y_{\text{ref}}$ based on the comparison result. We find that the former design performs better than the latter. We also compared with stepwise DPO \citep{kim2024sdpo}, which updates the reference model at each iteration but uses the original dataset instead of the updated one. This demonstrates that exploring and collecting new data is necessary.

\begin{figure}[H]
    \centering
    \begin{minipage}[b]{0.32\linewidth}
        \centering
        \includegraphics[width=\linewidth]{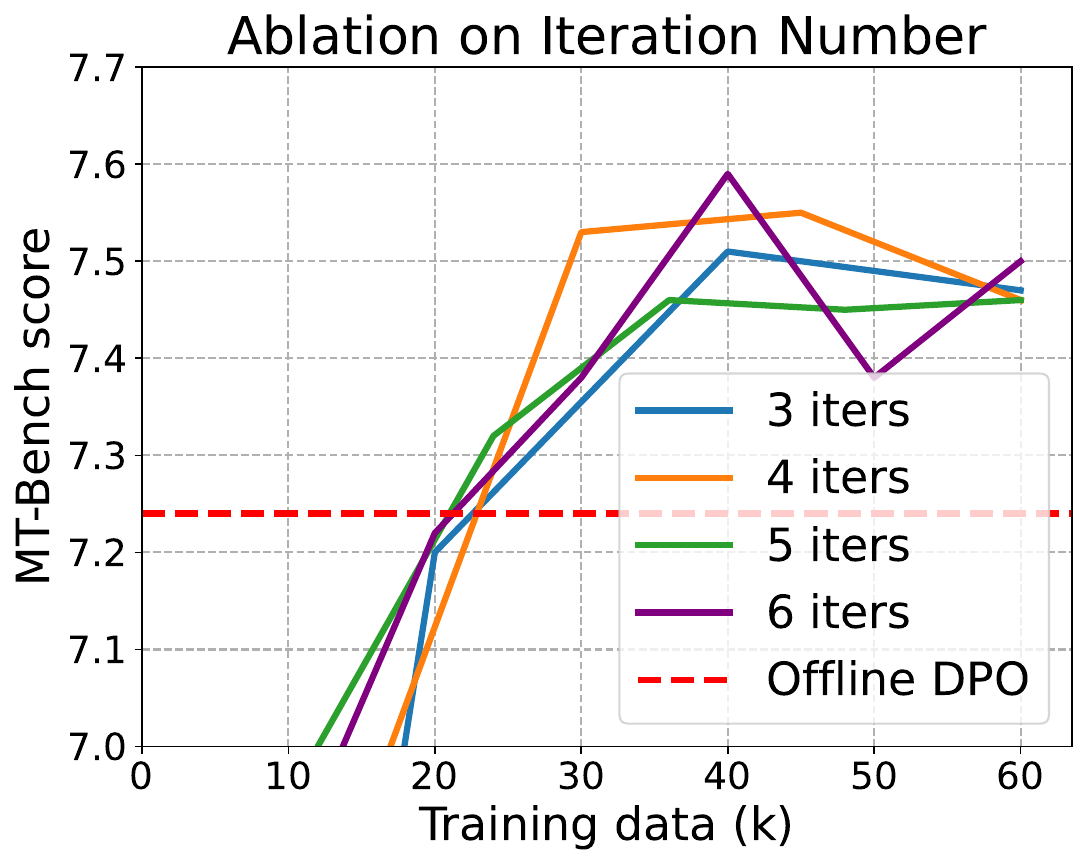}
    \end{minipage}
    \begin{minipage}[b]{0.32\linewidth}
        \centering
        \includegraphics[width=\linewidth]{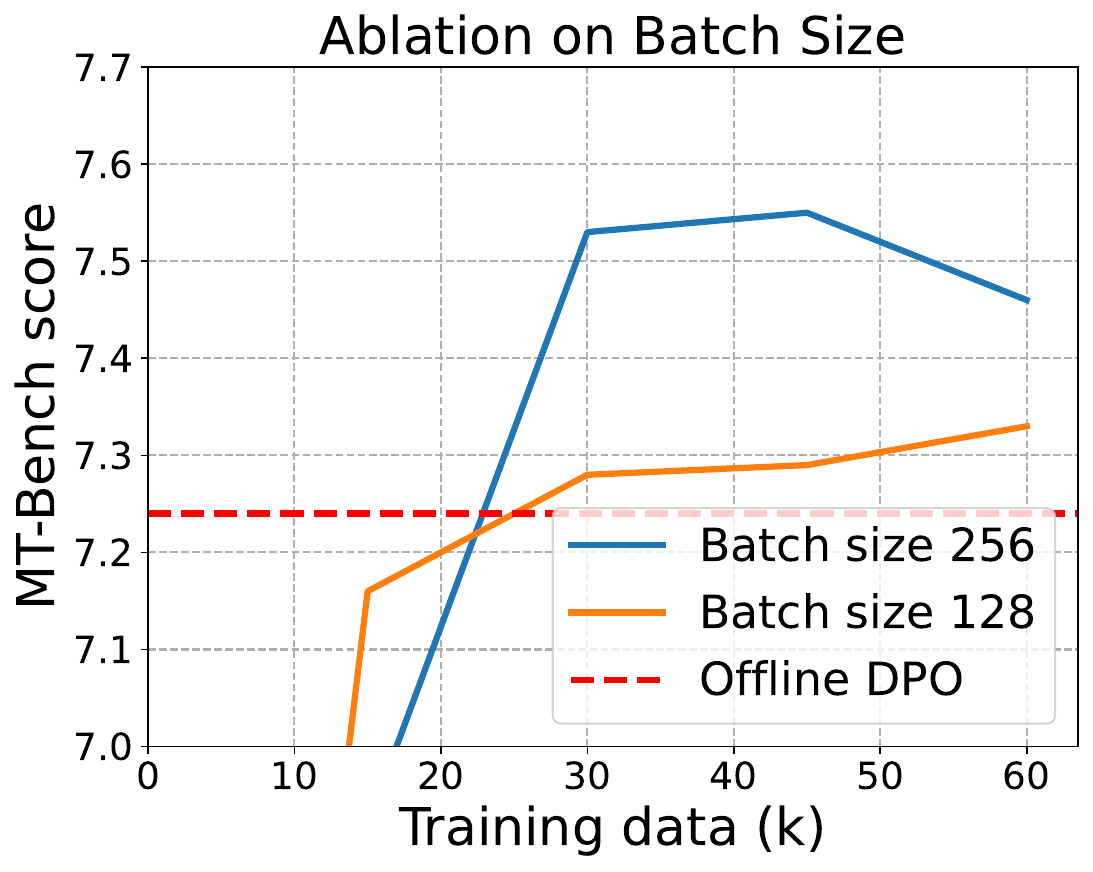}
    \end{minipage}
        \begin{minipage}[b]{0.32\linewidth}
        \centering
        \includegraphics[width=\linewidth]{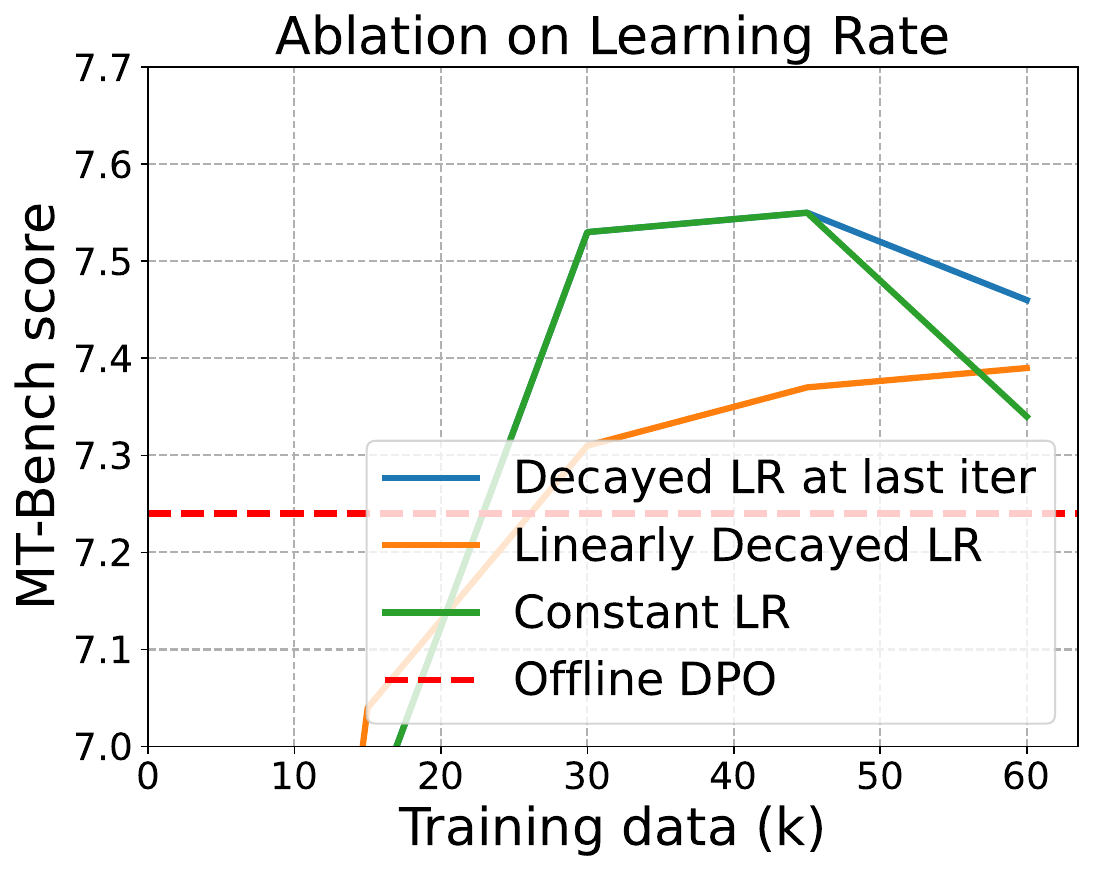}
    \end{minipage}
    \begin{minipage}[b]{0.32\linewidth}
        \centering
        \includegraphics[width=\linewidth]{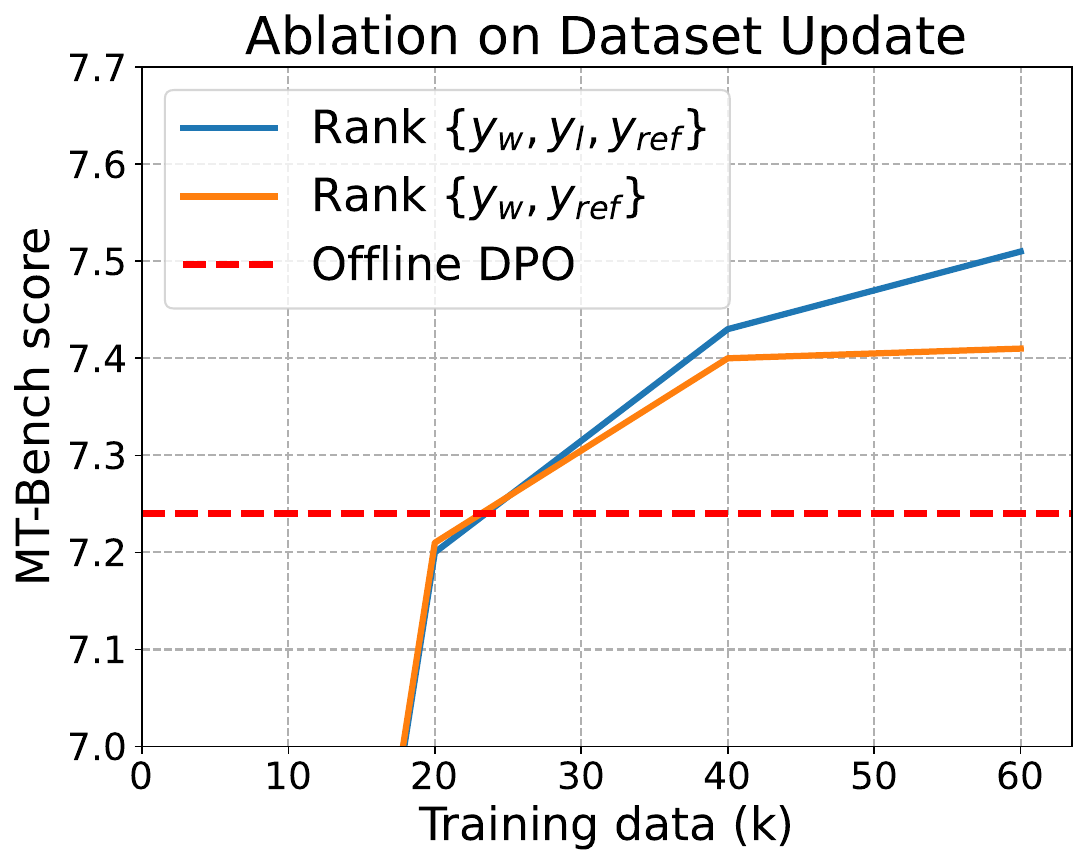}
    \end{minipage}
    \begin{minipage}[b]{0.32\linewidth}
    \centering
    \includegraphics[width=\linewidth]{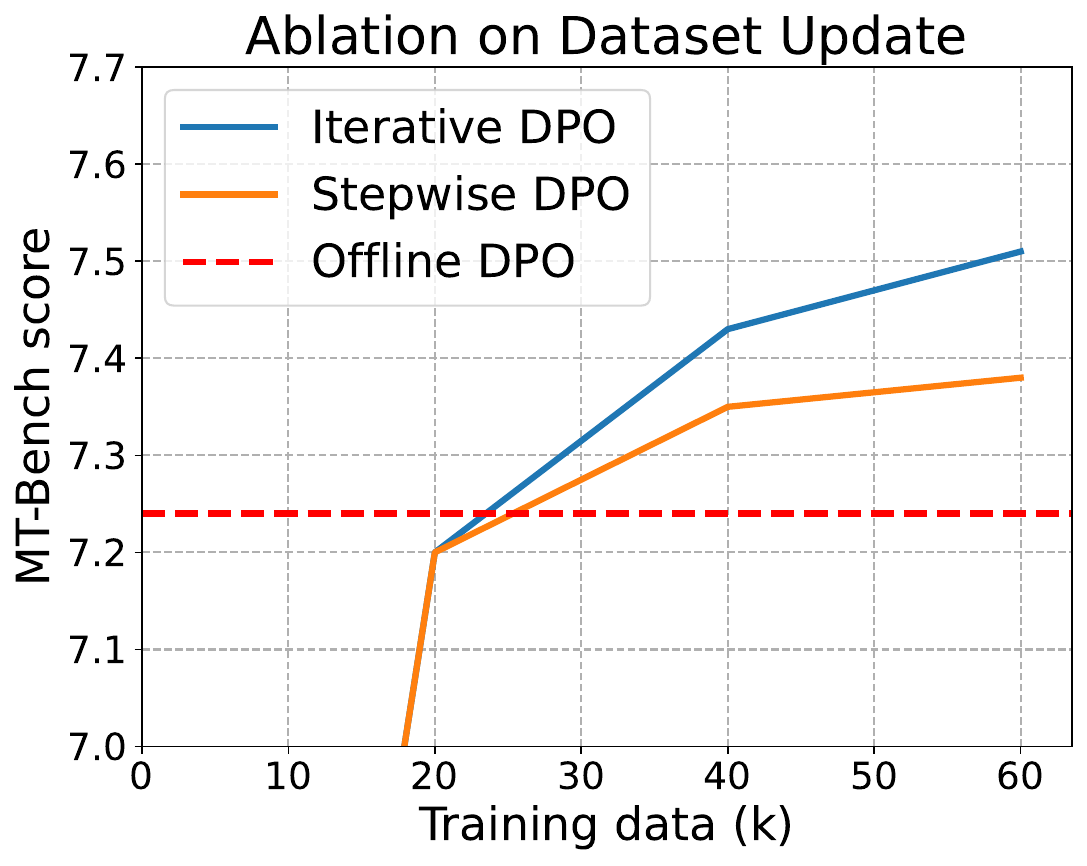}
    \end{minipage}
\caption{Ablation of the iterative DPO baseline. We conduct a grid search over the iteration number, batch size, learning rate, and designs of the dataset update rule.}
\label{fig_grid}
\end{figure}

For the proposed SELM method, we follow the above hyperparameter settings for a fair comparison. The optimism coefficient $\alpha$ is searched over $0.005$, $0.001$, $0.0005$, and $0.0001$ and is selected based on the average external reward on the holdout test set of UltraFeedback. We set $\alpha = 0.001$ for Zephyr-based SELM and $\alpha=0.0001$ for Llama3-It-based SELM. For training SELM based on other models, we recommend setting $\alpha = 0.005$ or $0.001$ as it shows minimal sensitivity to variations.

\end{document}